\documentclass[10pt]{article} 
\usepackage[preprint]{tmlr}

\usepackage{amsmath,amsfonts,bm}

\def\eqref#1{equation~\ref{#1}}

\def\1{\bm{1}}

\DeclareMathAlphabet{\mathsfit}{\encodingdefault}{\sfdefault}{m}{sl}
\SetMathAlphabet{\mathsfit}{bold}{\encodingdefault}{\sfdefault}{bx}{n}

\DeclareMathOperator*{\argmax}{arg\,max}

\usepackage{hyperref}
\usepackage{url}
\usepackage{amssymb} 
\usepackage{amsthm} 
\usepackage{graphicx}  
\usepackage{enumitem}  
\usepackage{thmtools}
\usepackage{subcaption}
\usepackage{thm-restate}

\declaretheorem[style=plain, numberwithin=section]{theorem}
\declaretheorem[style=plain, sibling=theorem]{proposition}
\declaretheorem[style=plain, sibling=theorem]{lemma}

\declaretheorem[style=definition, sibling=theorem]{definition}

\usepackage[most]{tcolorbox}

\newtcolorbox{summarybox}{
  colback=gray!10,
  colframe=black,
  boxrule=0.8pt,
  arc=3pt,
  left=8pt,
  right=8pt,
  top=6pt,
  bottom=6pt
}

\hypersetup{
  colorlinks=true,
  linkcolor=blue!50!black,
  citecolor=blue!50!black,
  urlcolor=blue!50!black
}

\title{Reward Learning from Best-of-$N$ Preference Data:\\
 Targets, Tradeoffs, and Design Principles}

\author{\name Rattana Pukdee \email rpukdee@cs.cmu.edu \\
      \addr Machine Learning Department\\
      Carnegie Mellon University
      \AND
      \name Maria-Florina Balcan \email ninamf@cs.cmu.edu \\
      \addr Machine Learning Department\\
      Carnegie Mellon University
      \AND
      \name Pradeep Ravikumar \email pkr@cs.cmu.edu\\
      \addr Machine Learning Department\\
      Carnegie Mellon University}

\def\month{MM}  
\def\year{YYYY} 
\def\openreview{\url{https://openreview.net/forum?id=XXXX}} 

\newcommand{\pBoN}{P_{\text{BoN}}}
\newcommand{\pBoNNminus}{P_{\text{BoN}, N-1}}
\newcommand{\pWoN}{P_{\text{WoN}}}

\newcommand{\lambdaconn}{\lambda_{\text{conn}}}
\newcommand{\Deltar}{\Delta_{r}}
\newcommand{\Deltarstarp}{\Delta_{r^*_P}}
\newcommand{\Deltarstar}{\Delta_{r^*}}

\newcommand{\DeltarBI}{\Delta_{r^*_{\text{BI}}}}
\newcommand{\DeltarBWI}{\Delta_{r^*_{\text{BWI}}}}
\newcommand{\cH}{\mathcal{H}}
\newcommand{\cX}{\mathcal{X}}
\newcommand{\cY}{\mathcal{Y}}

\newcommand{\bE}{\mathbb{E}}

\newcommand{\rstar}{r^*}
\newcommand{\pbase}{P_{\text{base}}(y \mid x)}
\newcommand{\pbaseyprime}{P_{\text{base}}(y' \mid x)}
\newcommand{\PBI}{P_{\text{BI}}}
\newcommand{\PBR}{P_{\text{BR}}}
\newcommand{\PBW}{P_{\text{BW}}}
\newcommand{\PBWI}{P_{\text{BWI}}}
\newcommand{\PBWIunnorm}{\widetilde{P}_{\text{BWI}}}
\newcommand{\PBWunnorm}{\widetilde{P}_{\text{BW}}}
\newcommand{\rhat}{\hat{r}}
\newcommand{\rstarbi}{r^*_{\text{BI}}}
\newcommand{\rstarbr}{r^*_{\text{BR}}}
\newcommand{\rstarbw}{r^*_{\text{BW}}}
\newcommand{\rstarbwi}{r^*_{\text{BWI}}}
\newcommand{\TV}{\operatorname{TV}}
\newcommand{\LP}{\mathcal{L}_P}
\newcommand{\LQ}{\mathcal{L}_Q}
\newcommand{\rstarp}{r^*_{P}}
\newcommand{\rstarq}{r^*_{Q}}
\newcommand{\LBR}{\mathcal{L}_{\PBR}}
\newcommand{\LBI}{\mathcal{L}_{\PBI}}
\newcommand{\Wp}{\widetilde{P}}
\newcommand{\Wq}{\widetilde{Q}}
\newcommand{\PBIunnorm}{\widetilde{P}_{\text{BI}}}
\newcommand{\PBRunnorm}{\widetilde{P}_{\text{BR}}}
\newcommand{\cHpair}{\mathcal{H}_\text{pair}}
\newcommand{\aalphax}{A_{x, \alpha}}

\begin{document}

\maketitle

\begin{abstract}
  Best-of-$N$ sampling is widely used to construct pairwise preference data: $N$ candidates are drawn from a base distribution, and the best is paired with a rejected response. Despite its widespread use, what Bradley--Terry (BT) reward learning extracts from such data, and how to choose $N$ and the base distribution, remain unclear. We specialize a recent analysis of preference data via its induced conditional distribution to Best-of-$N$. For independent-reference variants, we derive closed-form reward targets as explicit functions of $N$ and the base distribution, and show that they preserve the latent reward ranking. For the practical Best-vs-Random and Best-vs-Worst variants, chosen and rejected responses are coupled through the same candidate set, so exact BT representability generally fails; nevertheless, bounded-class minimizers approach the reference targets as $N$ grows. Although margin and connectivity are known to govern sample efficiency in pairwise preference learning, Best-of-$N$ couples them through $N$ in opposing directions: larger $N$ widens pairwise margins but reduces connectivity. This trade-off yields two design principles: use larger $N$ when preference labels are the bottleneck, smaller $N$ when generation is the bottleneck; and shape the base distribution to place mass between the responses whose comparison matters most at test time. Experiments on synthetic and real preference data support the predicted dependence on sample size and base-distribution shape. \footnote{Code available at: \href{https://github.com/rattaoup/bon-preference}{https://github.com/rattaoup/bon-preference}}
\end{abstract}

\section{Introduction}

Pairwise preference data is a general and increasingly important supervision signal in machine learning. It has been extensively studied in the literature on learning to rank \citep{cao2007learning, wauthier2013efficient} and learning preference \citep{furnkranz2010preference, qian2015learning}, learning combinatorial functions \citep{balcan2016learning}  and more recently has become central to modern AI systems, especially in reward learning and post-training of large language models \citep{christiano2017deep, ouyang2022training}. This form of supervision is widely used because relative judgements are often easier to obtain, from both human annotators and increasingly AI evaluators \citep{zheng2023judging}, than calibrated scalar scores. 

In modern AI systems, pairwise preference data is often used to train reward models, typically with a Bradley-Terry objective, and these reward models may then be used to further optimize a language model through reinforcement learning \citep{ziegler2019fine, stiennon2020learning, ouyang2022training}. Preference data is also used more directly in fine-tuning the model as well \citep{rafailov2023direct, meng2024simpo}. A common approach to construct the preference data is through Best-of-$N$ sampling. For a given context, $N$ responses are sampled from a base distribution, and an oracle identifies the best response as the chosen one. The rejected response is then selected either uniformly from the remaining $N - 1$ responses (Best-vs-Random) or as the worst response (Best-vs-Worst). The special case $N = 2$ recovers the standard pairwise-comparison setting \citep{christiano2017deep}. The common intuition is that larger $N$ improves this data, since the chosen response is better in absolute terms. We show this intuition is incomplete: what matters is not only the quality of the chosen response, but which comparisons the dataset makes available. We study reward learning as a clean lens on the role of data with two fundamental questions.

\begin{enumerate}
  \item When preference data is generated through Best-of-N sampling, what reward function is induced by the resulting dataset under Bradley-Terry objective?
  \item How should the candidate-set size N and the base distribution be chosen ?
\end{enumerate}

Our analysis builds on a recent distribution-agnostic framework \citep{pukdee2026doespreferencelearningrecover} that analyzes preference data through the conditional preference distribution (CPRD), characterizing when it is Bradley-Terry representable and showing that finite-sample learning is governed by margin and comparison connectivity. That framework leaves the induced reward targets, margins, and connectivity abstract, since they depend on the data-generation mechanism. We instantiate it for Best-of-$N$ preference data, where the order-statistic structure makes these quantities concrete and reveals how they scale with $N$. Our contributions are summarized as follows:

\textbf{Contributions.}
\begin{enumerate}
  \item \textit{Targets:} We study a clean reference setting in which the rejected response is sampled independently. In this setting, we characterize the reward targets induced by Best-of-$N$ preference data in closed form and show that they preserve the ranking of an underlying ground-truth reward. This identifies what signal reward learning is recovering from the data. For the practical Best-vs-Random and Best-vs-Worst settings, we show that the corresponding population minimizers do not admit the same clean characterization, but approach these reference targets as $N$ grows.
  \item \textit{Tradeoffs:}  We show that $N$ controls a fundamental tradeoff between margin and comparison connectivity. Increasing $N$ pushes chosen responses toward the top of the reward distribution, which amplifies pairwise margins and makes individual comparisons more informative. But this same concentration narrows the range of responses being compared, reducing comparison connectivity and leaving intermediate comparisons underrepresented. These two effects pull in opposite directions and govern sample efficiency.
  \item \textit{Design principles:} We derive two actionable guidelines. The candidate-set size $N$ should be chosen based on the resource bottleneck: larger $N$ is more helpful when preference labels are scarce, whereas smaller $N$ is often preferable when preference labels are cheap and data is abundant. Beyond choosing $N$, the base distribution itself is also a design lever through what we call the between-mass principle: the induced margin between two responses grows with the amount of base-distribution mass placed between them in reward order. Shaping the base distribution to put mass between the response pairs that matter under the target evaluation distribution selectively amplifies those margins. We validate these predictions on both synthetic and real preference data. 
\end{enumerate}

Taken together, our results provide a learning-theoretic foundation for reward learning from Best-of-$N$ preference data, while yielding concrete principles for data design in practice.

\section{Related work}

\textbf{Best-of-$N$ Sampling for Inference and Training.} Best-of-$N$ sampling (BoN) is an inference-time method with strong empirical performance \citep{stiennon2020learning, cobbe2021training, nakano2021webgpt, 
brown2024large}. The idea is to sample $N$ responses from a policy and select 
the one with the highest reward. Despite its simplicity, BoN is a strong baseline 
for inference-time scaling \citep{snell2024scaling}. Various works have studied the theoretical properties of the induced BoN policy \citep{huang2025sample}, including its distance from the reference policy and its win rate  \citep{beiramitheoretical, yang2024asymptotics}, and the soft BoN variants \citep{verdun2025soft, aminianbest}. These results suggest that win rate increases 
with $N$. However, larger $N$ does not always improve performance in practice; 
when $N$ is too large, reward hacking can degrade output quality 
\citep{gao2023scaling, huangbest}. Given the effectiveness of BoN but its $N$-fold inference cost, several works aim  to train a policy that matches the BoN distribution while requiring only a single  sample at inference time. This can be done offline with a fixed BoN dataset 
\citep{touvron2023llama, dongraft, aminivariational} or iteratively \citep{sessabond, yang2025faster, guo2024direct}. \citet{gui2024bonbon} provide theoretical analysis for this training time setting, showing that the BoN distribution is optimal in terms of win rate against the base model for a given KL distance budget. 
While this literature clarifies the role of BoN in improving generation quality and policy training, it leaves open the question of how BoN affects learning when it is used to construct preference data.

\textbf{Learning from Best-of-$N$ Preference Data.}
Understanding how preference data affects the final model is an important question, since it can inform the design of better datasets. Recent work suggests that better preference data leads to better downstream policies \citep{ivison2024unpacking,morimura2024filtered,pan2025matters}. In this literature, however, better preference data is often identified using external proxies such as reward-model scores. While a higher score may indicate that a response is better in an absolute sense, it does not necessarily mean that the resulting pairwise comparison is more informative for preference learning. We study a related question from a different perspective: rather than analyzing the end-to-end policy outcome, we stop at the reward-learning stage and ask how the preference data shapes the learned reward function. This provides a cleaner view of the effect of the data itself, without conflating it with the base policy's behavior. \citet{razinmakes} study the complementary direction, linking reward properties to downstream policy performance, closing the loop from data to policy.


 In practice, preference data can be constructed in several ways, including by comparing responses from language models of different capacities \citep{kim2023aligning}, using different prompts to produce chosen and rejected responses \citep{yang2023rlcd}, or, most commonly, by sampling multiple responses from the same base model and using a human annotator, reward model, or model-based judge to identify preferred and dispreferred responses from the candidate set \citep{bai2022constitutional,guo2024direct,lee2024rlaif}. In this paper, we focus on the last setting, namely Best-of-$N$ preference data. Prior work has studied this setting empirically for reward learning: \citet{pace2024west} show that selecting the best and worst candidates from a pool can improve reward-model training, while \citet{yuan2024self} study settings in which the language model itself is also used to provide reward feedback. We ask a more foundational question: what reward function is induced when pairwise comparisons are generated through Best-of-$N$ sampling?

\section{Problem Setup}

In practice, Best-of-$N$ (BoN) preference data pipelines may rely on human annotators, reward models, or LLM judges to identify preferred responses. We study a clean theoretical abstraction of this setting in which preferences are induced by an underlying latent reward function $r^*$, even though $r^*$ itself is never observed. This abstraction allows us to ask what reward function is learned from the resulting comparison data.

Let $\cX$ be a context space and $\cY$ be a set of outcomes. Assume that $\cX$ and $\cY$ are finite but can be exponentially large. Let $\rstar: \cX \times \cY \to \mathbb{R}$ be a target latent reward function. The supervision signal is in the form of pairwise comparison data where triplets $(x, y^+, y^-)$ indicate that a chosen outcome $y^+$ is preferred over a rejected outcome $y^-$ given a context $x$ i.e. the latent reward $r^*(x, y^+) > r^*(x, y^-)$.

To construct the pairwise preference data, the chosen outcome $y^+$ is sampled from the Best-of-$N$ mechanism. First, we sample $N$ outcomes from a base distribution $y_1, \dots, y_N \sim P_{\text{base}}(\cdot \mid x)$, the outcome with the highest reward score is the chosen outcome, $y^+ = \arg\max_{y_i} r^*(x, y_i)$. This procedure only requires access to an oracle that can identify the best outcome among $N$ candidates, rather than access to the reward scores $r^*(x, y_i)$ themselves. This is a common approach in practice to construct pairwise preference data. For the rejected outcome $y^-$, two dominant approaches are
\begin{itemize}
    \item \textbf{Best-vs-Random} \citep{khaki2024rs, ivison2024unpacking}: we sample $y^-$ from $\{y_1, \dots, y_N\} \setminus \{y^+\}$ uniformly at random.
    \item \textbf{Best-vs-Worst} \citep{pace2024west, yuan2024self}: the outcome with the lowest reward score among $\{y_1, \dots, y_N\}$ is the rejected outcome, $y^- = \arg\min_{y_i} r^*(x, y_i)$. This may require another oracle call to identify the worst outcome.
\end{itemize}
Whenever we observe a collision $y^+ = y^-$, we discard the triplet. Best-vs-Random and Best-vs-Worst guarantee that the latent reward score $y^+$ is always higher than $y^-$. 
Let $P$ be the probability distribution of the BoN preference data. Let $\mathcal{H}$ be a hypothesis class of reward functions. A common learning objective for reward learning is Bradley-Terry (BT) objective given by
\begin{equation}
    \label{eq: population-risk}
    \mathcal{L}_P(r) = \mathbb{E}_{P} \left[ - \log \sigma(r(x,y^+) - r(x,y^-)) \right]
\end{equation}
when $\sigma(t) = \frac{1}{1+e^{-t}}$ is the sigmoid function. Since the loss only depends on the difference between the reward scores of the chosen and rejected outcomes, $r(x,y^+) - r(x,y^-)$, the reward is identifiable only up to an additive constant. To fix this ambiguity, we restrict $\mathcal{H}$ to normalized reward functions satisfying $\mathbb{E}_{(x,y) \sim \mu}[r(x,y)] = 0$ for a fixed reference distribution $\mu$ over $\cX \times \cY$. For example, this can be a uniform distribution. Finally, we define some notations that will be useful for our analysis. First, score quantiles $F(x,y)$ capture the probability that a random outcome from the base distribution $P_{\text{base}}(\cdot \mid x)$ has a lower or strictly lower reward score than the reward of $y$.

\begin{definition}[Strict and Inclusive Score Quantiles]
    For any outcome $y\in \cY$, define the \emph{strict quantile}
    \begin{equation}
    \label{eq:strict-quantile}
        F^{-}(x,y)
        := \Pr_{y'\sim P_{\text{base}}(\cdot \mid x)}[\,r^*(x,y') < r^*(x,y)\,]
    \end{equation}
    as the probability that a randomly drawn outcome from the base distribution has strictly lower reward than $y$. We also define the \emph{inclusive quantile}
    \begin{equation}
    \label{eq:inclusive-quantile}
        F(x,y)
        := \Pr_{y'\sim P_{\text{base}}(\cdot \mid x)}[\,r^*(x,y') \le r^*(x,y)\,]
        = F^{-}(x,y) + \pbase.
    \end{equation}
    \end{definition}

We also define pairwise margin,
\begin{definition}[Margin]
    For any reward function $r$ and any context $x$ and outcomes $y, y'$,  the pairwise margin of $r$ is given by
    \begin{equation}
        \label{eq:pairwise-margin}
            \Delta_r(x;y,y') := r(x,y)-r(x,y')
        \end{equation}
    \end{definition}

With these definitions in hand, we turn to our first question: what reward function does BoN preference data induce?

\section{Characterizing the Learned Reward}
\label{sec:characterizing-learned-reward}

We begin with the infinite data setting and analyze the minimizer of the population risk (\eqref{eq: population-risk}). For analysis, we first define reference variants \textbf{Best-vs-Independent} and \textbf{Best-vs-Worst-Independent}, where the negative sample is drawn independently from the base distribution and the Worst-of-N sampling respectively. These variants have a much cleaner analysis and we will later on connect it back to the practical settings of Best-vs-Random and Best-vs-Worst. We will use subscript $BR, BW, BI, BWI$ to denote if the data distribution and the reward function corresponds to Best-vs-Random, Best-vs-Worst, Best-vs-Independent and Best-vs-Worst-Independent respectively.

\subsection{Independent Rejected Outcome}

We start with the Best-vs-Independent scenario. The rejected outcome is independent of the chosen one, we can factorize the probability mass of $\PBI$ as
\begin{equation}
    \PBI(x, y^+, y^-) = P(x) \pBoN(y^+ \mid x) P_{\text{base}}(y^- \mid x)
\end{equation}
where $\pBoN(y \mid x)$ is the conditional distribution of the outcome sampled with the Best-of-N sampling. We can derive $\pBoN(y \mid x)$ in terms of score quantiles.

\begin{restatable}[]{lemma}{BoNDistribution}
\label{lemma: bon-closed-form}    
Assuming that there is no tie in reward score, for any context $x$ and outcome $y$, the probability mass of $\pBoN$ is given by
    \begin{equation}
        \label{eq:bon-closed-form}
        \pBoN(y \mid x) = F(x,y)^N - F^{-}(x,y)^N
        \end{equation}
    \end{restatable}

Positive-negative conditional independence under $\PBI$ is useful because \citet{pukdee2026doespreferencelearningrecover} showed that it is exactly the condition under which the induced pairwise preference admits a Bradley--Terry representation. In other words, the BT form is not merely a modeling assumption here, but the true functional form of $\Pr_{\text{BI}}(y \succ y' \mid x)$, the conditional probability of preferring $y$ over $y'$ given $x$:
\begin{equation}
\Pr_{\text{BI}}(y \succ y' \mid x) := \frac{\PBI(x,y,y')}{\PBI(x,y,y') + \PBI(x,y',y)} = \sigma(\rstarbi(x,y) - \rstarbi(x,y'))
\end{equation}
when $\rstarbi$ is the population minimizer. \citet{pukdee2026doespreferencelearningrecover} also showed that the minimizer is given by the log ratio between the distribution of the chosen and rejected responses. These results allow us to characterize the target reward function $\rstarbi$ under $\PBI$.

\begin{restatable}[Identification of $\rstarbi$]{theorem}{Identificationrstarbi}
    \label{thm:identification-rstarbi}
Let $ \rstarbi \in \arg\min_{r} \mathcal{L}_{\PBI}(r)$ then there exists a function $c : \cX \to \mathbb{R}$ such that
\begin{equation}
    \rstarbi(x,y) = \log \frac{\pBoN(y \mid x)}{P_{\text{base}}(y \mid x)} + c(x)
\end{equation}
for almost every pair $(x,y)$ with respect to $\PBI$. Furthermore, $\rstarbi$ preserves the ranking induced by $\rstar$ that is for any $x, y, y'$ with non-zero probability mass,
\begin{equation} 
    \label{eq:bon-ranking-preservation}
    r^*(x, y) > r^*(x, y') \implies \rstarbi(x, y) > \rstarbi(x, y')
    \end{equation}
\end{restatable}

In other words, even though we never observe $r^*$, minimizing the BT loss on BoN data recovers a reward function that agrees with $r^*$ on every pairwise comparison. Similarly, for the Best-vs-Worst-Independent, by conditional independence, we can also factorize the distribution $\PBWI$ as 
\begin{equation}
    \PBWI(x, y^+, y^-) = P(x) \pBoN(y^+ \mid x) \pWoN(y^- \mid x)
\end{equation}
and derive a similar characterization for $\rstarbwi$.
\begin{restatable}[]{lemma}{WoNDistribution}
    \label{lemma: won-closed-form}
    Assume that there is no tie in reward score, for any context $x$ and outcome $y$, the probability mass of $\pWoN$ is given by
    \begin{equation}
        \label{eq:won-closed-form}
        \pWoN(y \mid x) = (1 - F^{-}(x,y))^N - (1 - F(x,y))^N
        \end{equation}
\end{restatable}
\begin{restatable}[Identification of $\rstarbwi$]{theorem}{Identificationrstarbwi}
    \label{thm:identification-rstarbwi}
Let $ \rstarbwi \in \arg\min_{r} \mathcal{L}_{\PBWI}(r)$ then there exists a function $c : \cX \to \mathbb{R}$ such that
\begin{equation}
    \rstarbwi(x,y) = \log \frac{\pBoN(y \mid x)}{\pWoN(y \mid x)} + c(x)
\end{equation}
for almost every pair $(x,y)$. Furthermore, $\rstarbwi$ preserves the ranking induced by $\rstar$ that is for any $x, y, y'$ with non-zero probability mass,
\begin{equation} 
    \label{eq:bon-ranking-preservation-bwi}
    r^*(x, y) > r^*(x, y') \implies \rstarbwi(x, y) > \rstarbwi(x, y')
    \end{equation}
\end{restatable}

\subsection{Best-vs-Random and Best-vs-Worst}

Next, we turn to the Best-vs-Random and Best-vs-Worst scenarios where the chosen outcome and rejected outcome are no longer independent. Without independence, we show that the population risk minimizer with respect to $\PBR$ and $\PBW$ does not have a closed-form minimizer anymore. We introduce a notion of comparison distribution to help with the analysis. This captures how often a pair of outcomes are compared together for a given context, regardless of which is preferred.
\begin{definition}[Comparison distribution]
    For any distribution $P$ over $\mathcal{X} \times \mathcal{Y} \times \mathcal{Y}$, the comparison distribution $\widetilde P$ is a distribution over $x$ and an unordered pair $\{y,y'\}$ with $y\neq y'$ with density
    \begin{equation}
        \widetilde P(x, \{y,y'\} ) \propto P(x,y,y') + P(x,y',y)
    \end{equation}
    \end{definition}

    \begin{restatable}[Best-vs-Random and Best-vs-Worst perform pairwise margin maximization]{proposition}{ObjectiveForBRandBW}
        Assume that there is no tie in reward score.
        For each context and pair of outcomes $(x,\{y,y'\})$, define
        \[
        y_{\max} \in \arg\max_{u\in\{y,y'\}} r^*(x,u),
        \qquad
        y_{\min} \in \arg\min_{u\in\{y,y'\}} r^*(x,u)
        \]
        Then for any hypothesis class $\cH$, the population risk minimizer $\rstarbr \in \arg\min_{r \in \cH} \mathcal{L}_{\PBR}(r)$ satisfies
        \[
        \rstarbr \in \arg\min_{r\in\cH}\;
        \mathbb{E}_{(x,\{y,y'\})\sim \widetilde P_{\mathrm{BR}}}
        \Big[-\log \sigma\!\big(r(x,y_{\max})-r(x,y_{\min})\big)\Big]
        \]
        The same result holds for $\rstarbw$ by replacing the comparison distribution $\widetilde P_{\mathrm{BR}}$ with $\widetilde P_{\mathrm{BW}}$.
        In particular, both objectives are the same pairwise margin-maximization loss, differing only in the comparison distribution.
        \end{restatable}

The proposition shows that for any pair of outcomes $(\{y,y'\})$, the BT objective maximizes the margin between the outcomes with higher and lower true reward scores. As a result, if $\cH$ is unbounded, then we can decrease the BT objective by increasing the margin on any pair indefinitely, $\Delta_r(x, y_{\max}, y_{\min}) \to \infty$, which leads to an unbounded minimizer. As such, we do not have a nice closed-form minimizer for $\rstarbr$ and $\rstarbw$, and the minimizer depends on properties of $\cH$. Regardless, we can still show that $\rstarbr \to \rstarbi$ and $\rstarbw \to \rstarbwi$ in margin semi-norm as $N$ grows.

\begin{definition}[Margin semi-norm]
    For a preference distribution $P$, the margin semi-norm between two reward functions $r_1$ and $r_2$ is defined as
\begin{equation}
\lVert r_1 - r_2 \rVert_{\Delta, P} = \mathbb{E}_{(x,y,y')\sim P} \left[ (\Delta_{r_1} - \Delta_{r_2})^2 \right]
\end{equation}

\end{definition}

  \begin{restatable}[$\rstarbr$ is close to $\rstarbi$]{theorem}{Qualityofrstarbr}
    \label{thm:quality-rstarbr}
    Let $\cH$ be a convex hypothesis class that is bounded by some constant $B$. Let $\rstarbr \in \arg\min_{r \in \cH} \mathcal{L}_{\PBR}(r)$  then  there exists a constant $c_B$ such that for any $N \geq 2$,
    \begin{equation}
        \lVert \rstarbr - \rstarbi \rVert_{\Delta, \PBR} \leq \frac{c_B}{N-1}
    \end{equation}
  \end{restatable}

  \begin{restatable}[$\rstarbw$ is close to $\rstarbwi$]{theorem}{Qualityofrstarbw}
    \label{thm:quality-rstarbw}
    Let $\cH$ be a convex hypothesis class that is bounded by some constant $B$. Let $\rstarbw \in \arg\min_{r \in \cH} \mathcal{L}_{\PBW}(r)$ then  there exists a constant $c_B$ and $\rho < 1$ such that for any $N \geq 2$,
    \begin{equation}
        \lVert \rstarbw - \rstarbwi \rVert_{\Delta, \PBW} \leq c_B\rho^N
    \end{equation}
 \end{restatable}

 The key step in the proof is to show that $\PBR$ and $\PBW$ are close in total variation distance to $\PBI$ and $\PBWI$, respectively. Establishing this requires deriving closed-form expressions for the probability mass functions of these distributions and then carrying out careful algebraic comparisons; this is the most technically involved part of the argument. 
 The convexity of $\cH$ is then used only to transfer this distributional closeness to closeness of the corresponding population minimizers, via a standard argument from learning theory. We defer the full proof to Appendix~\ref{app:characterizing-learned-reward}.

These results imply that, as $N$ grows, the margins induced by $\rstarbr$ and $\rstarbw$ become close to those induced by $\rstarbi$ and $\rstarbwi$, respectively. Since $\rstarbi$ and $\rstarbwi$ achieve perfect accuracy, and since the ranking is determined by the sign of the margin, it follows that the Best-vs-Random and Best-vs-Worst settings also approximately recover the ranking induced by $\rstar$. Having identified these as sensible population targets, we turn to the finite-sample setting, where a fundamental tradeoff emerges: increasing $N$ helps in one way, but hurts in another.

\textbf{Remark:} Unlike the independent variants, the pairwise preferences induced by $\PBR$ and $\PBW$ do not admit a BT form. However, the TV bounds above imply they approach those of the independent variants as $N$ grows, making BT an increasingly reasonable model.

\section{Sample Complexity}
\label{sec:bon-theory}

In practice, we only have access to a finite number of triplets, so even if the population risk minimizer is a good target, we must also ask how quickly we can learn it. We measure the performance of the learned reward in terms of the consistency of ordering with respect to the target reward function. We define the accuracy as follows.

\begin{definition}[Accuracy]
    Let $Q$ be a test distribution over $\cX \times \cY \times \cY$. The accuracy of a reward function $r$ with respect to $Q$ and $r^*$ is given by
    \begin{equation}
        \label{eq:accuracy-def}
        \operatorname{Acc}_Q(r) = \Pr_{\substack{x, y, y' \sim Q}} [\operatorname{sign}(\Deltar (x,y,y')) = \operatorname{sign}(\Deltarstar (x,y,y')) ]
    \end{equation}
    \end{definition}

Here, we fix the test distribution $Q$ to represent the downstream quantity of interest, while varying the training distribution $P$, which depends on $N$ and the base distribution $P_\text{base}$. This lets us assess the impact of the BoN preference distribution. We will focus on the Best-vs-Independent and Best-vs-Worst-Independent scenarios since they have clean target reward functions under BoN preference data $(\rstarbi, \rstarbwi)$ which preserve the ranking as $r^*$.  With finite samples, the learned reward $\hat{r}$ deviates from the target due to estimation error, and whether this error flips the ranking on a given pair depends on whether it exceeds the margin between them. \citet{pukdee2026doespreferencelearningrecover} formalized this intuition and showed that sample complexity is governed by two quantities: the margin and the connectivity of the comparison distribution. A larger margin between two pairs of outcomes implies that the correct ordering can tolerate a larger estimation error, and a larger connectivity leads to a smaller estimation error. We analyze how the BoN preference data shapes both quantities.

\subsection{Margin}

\begin{figure}[h]
    \centering
    \includegraphics[width=0.6\linewidth]{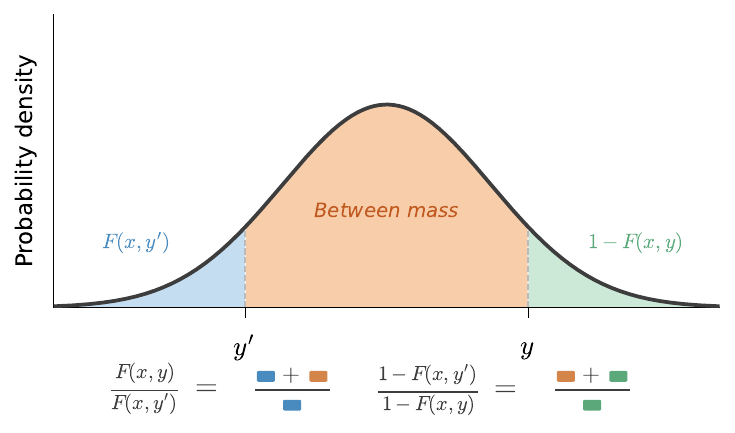}
    \caption{Illustration of the terms that contribute to the margin of $\rstarbi$ and $\rstarbwi$.}
    \label{fig:bon-margin-figure}
\end{figure}

First, we derive the margin of $\rstarbi$ and $\rstarbwi$ and show that they grow linearly with $N$.
\begin{restatable}[Margin of $\rstarbi$]{proposition}{Marginofrstarbi}
    \label{thm:margin-rstarbi}
Assuming no ties in reward score, for any context $x$ and pair of outcomes $y,y'$ with non-zero probability mass, the margin of $\rstarbi$ admits the decomposition
\begin{equation}
\label{eq:bon-margin-decomposition}
    \DeltarBI(x;y,y')
    ~=~
    N\log \frac{F(x,y)}{F(x,y')}
    ~+~ B_N(x;y,y'),
\end{equation}
where $B_N(x;y,y')$ is uniformly bounded in $N$. 
\end{restatable}

\begin{restatable}[Margin of $\rstarbwi$]{proposition}{Marginofrstarbwi}
    \label{thm:margin-rstarbwi}
Assume no ties in $r^*(x,\cdot)$.
For any context $x$ and pair of outcomes $y,y'$ with non-zero probability mass, the margin of $\rstarbwi$ admits the decomposition
\begin{equation}
\label{eq:bon-margin-decomposition-bwi}
    \DeltarBWI(x;y,y')
    = N \left(\log \frac{F(x,y)}{F(x,y')} + \log \frac{1 - F^{-}(x, y')}{1 - F^{-}(x,y)}\right) + B_N(x;y,y')
\end{equation}
where $B_N(x;y,y')$ is uniformly bounded in $N$.
\end{restatable}

The most immediate implication is that increasing $N$ amplifies all pairwise margins uniformly, making it a straightforward way to make the target reward easier to learn. But the slope at which the margin grows reveals a more informative and targeted lever. The slope increases when the base distribution places more mass between the two outcomes in the reward ranking (Figure \ref{fig:bon-margin-figure}). We define this quantity as the between mass.
\begin{definition}[Between Mass]
For any context $x$ and outcomes $y, y'$ where $r^*(x,y) > r^*(x,y')$, the between mass between $y$ and $y'$ is given by
\begin{equation}
\Pr_{\tilde{y} \sim P_{\text{base}}(\cdot \mid x)}(r^*(x,y') < r^*(x,\tilde{y}) < r^*(x,y)) = F(x,y) - F(x,y')
\end{equation}
\end{definition}

The higher the between mass between a pair of outcomes, the faster the margin between them grows with $N$.

So far, we have treated $P_{\text{base}}$ as fixed and studied how $N$ shapes the induced reward. The margin decomposition reveals a complementary lever: the base distribution itself enters the margin through the between-mass term. Unlike increasing $N$, which raises all margins indiscriminately, shaping $P_{\text{base}}$ allows one to selectively increase margins on specific pairs. This lever is available whenever the practitioner has some control over the generation distribution. For example, by choosing which model to sample from or applying filtering before BoN. This selectivity is appealing because, in practice, not all pairwise comparisons are equally important: in safety, comparisons that matter most are between harmful and acceptable outputs, in reasoning, the crucial distinctions are between correct solutions and plausible near-misses. Each setting induces a different test distribution, and our result provides a principled guide for designing the base distribution to increase margins precisely where they are needed. Concretely, we provide examples of how to design the base distribution to increase margins for three different applications for reward learning.

\paragraph{Designing the Base Distribution to Increase Margins}

\paragraph{Safety.}
The key comparisons are those that separate harmful responses from acceptable ones. To enlarge these margins, the base distribution should place substantial mass in the mediocre region: responses that are flawed or suboptimal, but not obviously dangerous. This creates more separation between harmful and acceptable behavior.

\paragraph{Reasoning.}
The critical distinction is between plausible-but-incorrect solutions and fully correct ones. If the base distribution produces too many clearly poor responses, then little mass lies between plausible and correct solutions. On the other hand, if it almost always produces perfect solutions, then this intermediate region again becomes small. Thus, the base distribution should be strong enough to reliably generate decent responses, while still concentrating substantial mass on near-correct but imperfect solutions.

\paragraph{General.}
These tasks require distinguishing among responses across a broad mid-to-high quality range. Since the relevant comparisons are distributed over a wider band, the base distribution should spread its mass throughout this region in order to preserve sufficient between-mass across many comparisons.


\subsection{Connectivity}
\label{sec:margin-connectivity-tradeoff}

\begin{figure}[t]
    \centering
    \includegraphics[width=0.70\linewidth]{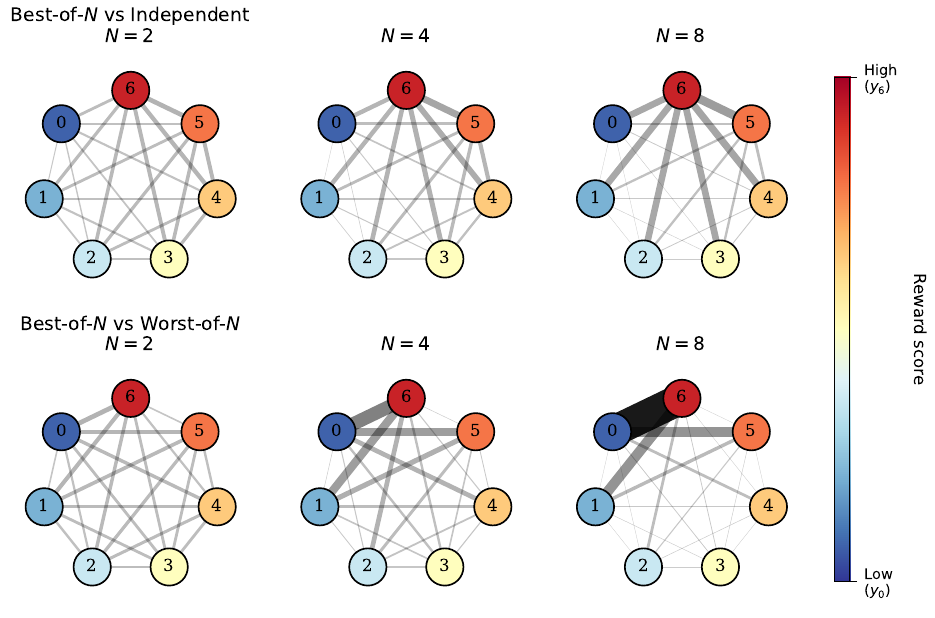}
    \caption{Comparison graph for Best-vs-Independent and Best-vs-Worst-Independent. It captures how often each pair of outcomes are compared in the BoN preference data.}
    \label{fig:comparison-graph}
\end{figure}

Recall from \citet{pukdee2026doespreferencelearningrecover}, the connectivity degree controls the estimation error of the reward learned from a training preference data distribution $P$ on a test distribution $Q$ in the worst-case scenario. It is defined as follows,

\begin{definition}[Connectivity]
The connectivity degree of a comparison distribution $P$ with respect to a hypothesis class $\mathcal{H}$ and a test distribution $Q$ is defined as
    \begin{equation}
        \label{eq:connectivity-def}
        \lambdaconn(P, Q ; \cH) = \inf_{f, g \in \cH} \frac{\bE_{P}[(\Delta f- \Delta g)^2]}{\bE_{Q}[(\Delta f- \Delta g)^2]}
    \end{equation}
    \end{definition}

The original notion of connectivity degree corresponds to the case where the test distribution $Q$ ranges over all comparisons in the outcome space, so that it measures how well the training distribution $P$ captures the pairwise variation over the entire comparison space. Here, we generalize this notion by allowing $Q$ to be an arbitrary, possibly selective test distribution. In this sense, $\lambdaconn(P, Q ; \cH)$ can also be interpreted as worst-case measure of how well $P$ covers the pairwise variation that matters under $Q$.

In contrast to the margin that grows linearly with $N$, we show that connectivity for $\PBI$ operates on an exponentially small scale in $N$. The intuition is that as $N$ increases, BoN sampling concentrates on outcomes with extremely high rewards while WoN sampling concentrates on outcomes with extremely low rewards. As a result, the observed comparisons become increasingly concentrated on a small subset of pairs, rather than covering a broad range of pairwise comparisons. Figure  \ref{fig:comparison-graph} illustrates this in a simple finite-outcome setting, where edge thickness represents the probability of observing each comparison. As $N$ increases, the Best-vs-Independent graph converges to a star graph, while the Best-vs-Worst-Independent graph converges to a single edge between the best and worst outcomes. This loss of comparison diversity leaves intermediate comparisons underrepresented, which in turn degrades connectivity with respect to the test distribution $Q$.

\begin{theorem}[Informal connectivity bounds for $\PBI$]
    \label{thm:connectivity-bound-rstarbi}
    Let $P_{\text{base}}$ be a base distribution. For any $\alpha \in \mathbb{R}$ and define $A_{x,\alpha} = \{y \in \cY : r^*(x, y) \leq \alpha\}$ as a set of outcomes which the reward is less than $\alpha$ for a given context $x$. If $\cH$ is expressive enough and $Q$ satisfies coverage and nondegeneracy conditions, there exist constants $C_Q, D_{Q,\alpha} > 0$ such that for any $N$,
    \[
        C_Q \Big(\min_{x,y} P_{\text{base}}(y \mid x)\Big)^N
        \leq
        \lambdaconn(\PBI, Q ; \cH)
        \leq
        D_{Q,\alpha} \bE_{x \sim P_x}\!\left[
            P_{\text{base}}(A_{x,\alpha} \mid x)^N
            +
            P_{\text{base}}(A_{x,\alpha} \mid x)
        \right].
    \]
    Thus, the connectivity degree for $\PBI$ is controlled by an exponentially small term in $N$, together with a residual term determined by the mass of the cut $A_{x,\alpha}$. 
\end{theorem}
Overall, large $N$ can reduce connectivity by concentrating probability mass on a narrower set of comparisons. The full formal statement and proof, along with a similar result for $\PBWI$, are given in Appendix \ref{app:sample-complexity}.

\subsection{Accuracy Bound}
From previous sections, $N$ is a single knob that controls both margin and connectivity degree. However, there is a trade-off, increasing $N$ leads to a larger margin but decreases the connectivity degree. To make this trade-off explicit, we combine the two results into an accuracy bound that yields a practical guideline for choosing $N$ as a function of the sample size $n$.

\begin{restatable}[Accuracy Bound]{theorem}{AccuracyBound}
    Assume that $\cH$ is convex, bounded by some constant $B$, and realizable, $\rstarbi \in \cH$. Assume also that for every triplet $(x,y,y')$ in the support of $Q$, both $P_{\text{base}}(y \mid x)$ and $P_{\text{base}}(y' \mid x)$ are non-zero. Let $\rhat$ be the empirical risk minimizer of the BT learning objective with respect to $\PBI$. Then there exist constants $D > 0$ and $M_Q > 0$ such that with probability at least $1 - \delta$,
    \begin{equation}\operatorname{Acc}_Q(\rhat) \geq
        \sup_{k > 0}
        \Big[
            \Pr_Q\left( \left|\log \frac{F(x,y)}{F(x,y')} \right| \geq k + \frac{M_Q}{N}\right)
            \;-\;
            D \frac{ \operatorname{Comp}(n, \cH, \delta)}{N^2k^2\lambdaconn}
        \Big]
    \end{equation}
where $\lambdaconn := \lambdaconn(\PBI, Q ; \cH)$.
\end{restatable}
The complexity term $\operatorname{Comp}(n, \cH, \delta)$ is a standard Rademacher-complexity quantity which we define rigorously in Appendix \ref{app:sample-complexity}. The first term measures how often the absolute margin slope under $Q$ exceeds a threshold, up to a vanishing $O(1/N)$ correction. In contrast, the second term captures the estimation error and decays like $1/(N^2 \lambdaconn)$. The bound therefore preserves the same interpretation: increasing $N$ helps by enlarging the target margin, but hurts through the connectivity degree. The optimal choice of $N$ depends on the sample size $n$, which in turn depends on the resource bottleneck of preference label collection.

\paragraph{Principle for Selecting $N$ Based on Resource Bottlenecks}

\paragraph{Preference Label-Limited Regime.}
In many practical settings, obtaining preference labels requires human annotators, which severely limits the number of triplets $n$ that can be collected. In this small-$n$ regime, the complexity term is large, so it is often advantageous to use a larger $N$. Although a larger $N$ increases the cost of identifying the BoN sample, the overhead can be modest in some applications: in visual tasks, selecting the best item from $N$ images may be relatively easy. In addition, selecting the best item among $N$ candidates is substantially easier than producing a full ranking, and can be done with at most $\lceil \log_2 N \rceil$ pairwise comparisons via a simple tournament procedure. Thus, when human annotation is the dominant bottleneck, it is often worthwhile to spend more effort per comparison triplet in order to improve the quality of each comparison.

\paragraph{Generation-Limited Regime.}
When preference annotations are cheap or can be automated, for example with LLM judges, it becomes possible to collect many triplets. In this large-$n$ regime, increasing $N$ provides diminishing returns because the complexity term is already small. It is therefore more compute-efficient to use a smaller $N$ and spend the saved generation budget on collecting more triplets, rather than on generating many candidates for each triplet.

\begin{summarybox}
    \textbf{Design Principles.}
    \begin{enumerate}
        \item When preference labels are scarce, choose a larger $N$.
        \item When preference labels are abundant, choose a smaller $N$ and use the saved generation budget to collect more triplets.
        \item When the base distribution is under the practitioner's control, shape it to maximize between mass on the comparisons that matter most under the target application.
    \end{enumerate}
    \end{summarybox}

\section{Experiments}
We validate our theoretical findings on synthetic data and real-world data. We provide a brief overview of the experiment setup and results in the main text. For full details, please refer to the Appendix \ref{app:experiment-details}.

\subsection{The optimal choice of $N$ depends on the sample size $n$.}
Our theory identifies $N$ as a single knob that controls both margin and connectivity degree; a larger $N$ increases the margin but also hurts the connectivity. Our design principle suggests that we should select $N$ based on the sample size $n$. A larger $N$ can help when $n$ is small but when $n$ is large, a smaller $N$ can become preferable.

\subsubsection{Synthetic Data}

\begin{figure*}[h]
    \centering
    \begin{subfigure}[b]{0.48\linewidth}
        \includegraphics[width=\linewidth]{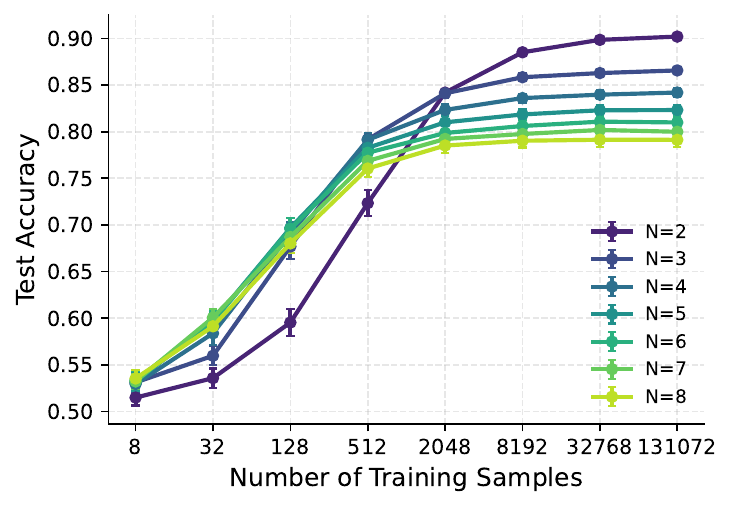}
        \caption{Accuracy}
    \end{subfigure}
    \hspace{0.1cm}
    \begin{subfigure}[b]{0.48\linewidth}
        \includegraphics[width=\linewidth]{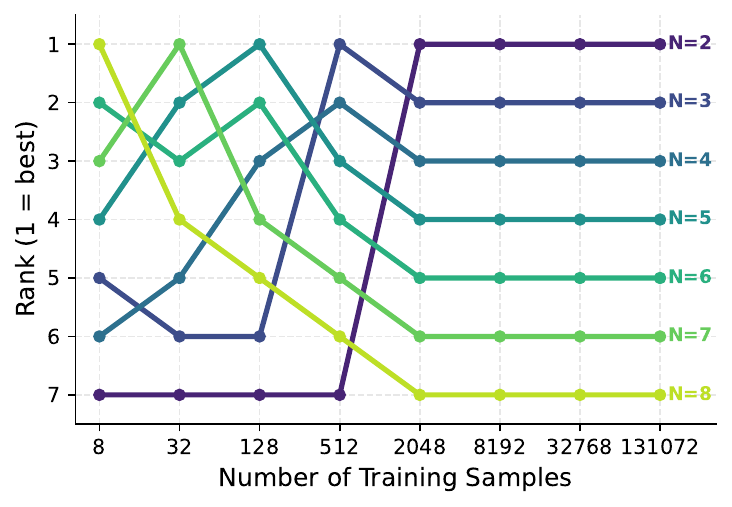}
        \caption{Ranking}
    \end{subfigure}
    \caption{Test accuracy for Best-vs-Independent on synthetic data over 20 random seeds.}
    \label{fig:bon-accuracy-vs-sample-size}
\end{figure*}

\paragraph{Setup.} We first validate this on a synthetic dataset following the setup in \citet{pukdee2026doespreferencelearningrecover}. In this setup, the ground-truth reward function is a cosine similarity between two embedding vectors where this embedding is an output of a two-layer neural network. We generate the training data by sampling $n$ triplets from the Best-vs-Independent setting and use it to train the reward model. Here, both base distribution and test distribution are uniform over the outcome space. 

\paragraph{Results.} Figure \ref{fig:bon-accuracy-vs-sample-size} confirms the predicted tradeoff between margin and connectivity. When $n$ is small, larger $N$ achieves higher accuracy; as $n$ increases, the optimal $N$ decreases. For sufficiently large $n$, smaller $N$ not only catches up to larger $N$ but can outperform it. In particular, for $n \geq 8192$, $N = 2$ matches or exceeds larger $N$, with a gap of up to 10 percent over $N = 8$ at $n = 131072$. This is the regime where the loss in connectivity from larger $N$ outweighs its margin benefit. We also plot the connectivity degree for each $N$, and additional Best-vs-Worst-Independent results in Appendix \ref{app:additional-results-synthetic}. The connectivity degree decreases as $N$ increases and the Best-vs-Worst-Independent setting shows the same pattern.

\subsubsection{Real-world Data}
\begin{figure}[h]
    \centering
    \begin{subfigure}[b]{0.48\textwidth}
        \centering
        \includegraphics[width=\textwidth]{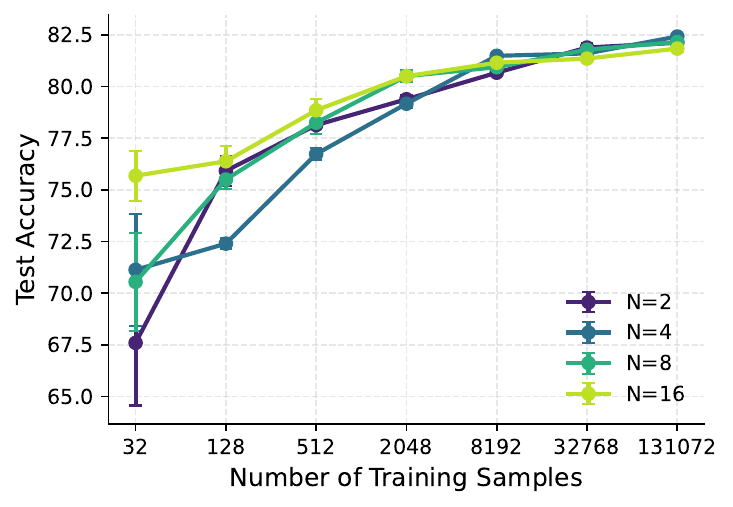}
        \caption{Best-vs-Random}
        \label{fig:accuracy-vs-n-best-vs-random}
    \end{subfigure}
    \hspace{0.1cm}
    \begin{subfigure}[b]{0.48\textwidth}
        \centering
        \includegraphics[width=\textwidth]{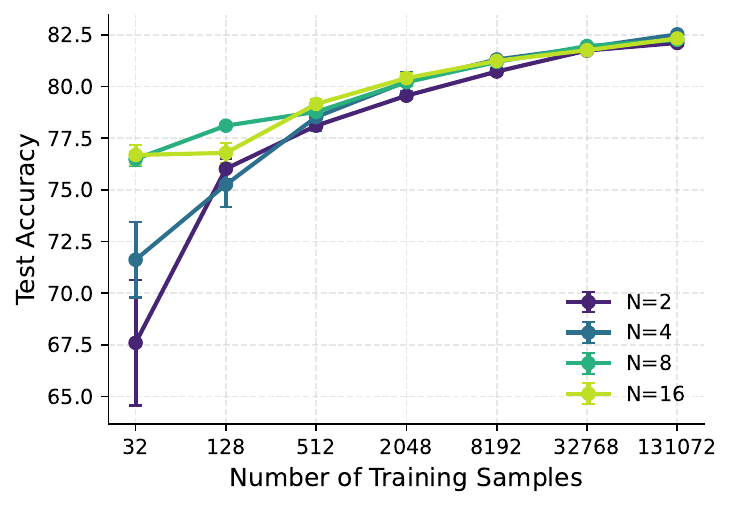}
        \caption{Best-vs-Worst}
        \label{fig:accuracy-vs-n-best-vs-worst}
    \end{subfigure}
    \caption{Accuracy vs. Sample size for each $N$ for the Best-vs-Random (left) and Best-vs-Worst (right) setting on UltraFeedback. The mean accuracy is computed over 3 random seeds and the standard error is reported.}
    \label{fig:accuracy-vs-N-ultrafeedback}
\end{figure}
\paragraph{Setup.}  We use prompts from \texttt{UltraFeedback} \citep{cui2024ultrafeedback} as contexts $x$, and generate responses with \texttt{Llama-3.1-8B-Instruct} \citep{grattafiori2024llama} under Best-vs-Random and Best-vs-Worst sampling. We use \texttt{Skywork-Reward-V2} \citep{liu2025skywork} as the target reward function, and evaluate on a binarized \texttt{UltraFeedback} test set whose preference labels follow this reward. The learned reward model is a sequence classifier built on a frozen \texttt{Llama-3.1-8B-Instruct} backbone with a linear reward head.

\paragraph{Results.} Figure \ref{fig:accuracy-vs-N-ultrafeedback} supports our design principle that the optimal choice of $N$ depends on the sample size $n$. When $n$ is small, larger $N$ yields higher accuracy. When $n$ is large enough, smaller $N$ achieves similar accuracy, allowing the saved compute to be used for generating more samples. For example, $n = 2048, N = 16$ achieves the same accuracy as $n = 8192, N = 2$, while the latter requires only half as much generation compute. In contrast to the synthetic data, we do not observe a crossover where smaller $N$ outperforms larger $N$ when $n$ is large. One possible explanation is that the test distribution here is not uniform over the outcome space, so connectivity may be less important than in the synthetic setting.

\paragraph{Additional Studies.} Appendix \ref{app:model-ablation} reports additional ablation studies. The connectivity degree decreases with $N$, consistent with our theory (Figure \ref{fig:connectivity-vs-n-ultrafeedback}). The Best-vs-Worst setting shows the same overall pattern and outperforms Best-vs-Random when $n$ is small, while the two perform similarly when $n$ is large. Across different base models, including \texttt{gemma-3-12b-it} \citep{kamath2025gemma}, \texttt{Qwen3-8B} \citep{qwen3technicalreport}, and \texttt{Ministral-3-8B-Instruct-2512} \citep{liu2026ministral}, the same broad pattern holds, although the variance is higher in the small-$n$ regime.

\subsection{Increasing the between mass of the base distribution can improve the accuracy. }

Beyond choosing $N$, our margin analysis suggests a second design principle: when the practitioner has control over the base distribution, $P_{\text{base}}$ should place more mass between the comparisons that matter under the target test distribution. We now test whether reshaping $P_{\text{base}}$ in this way translates into accuracy gains.


\paragraph{Setup.} 
We evaluate this design principle on 3 datasets representing safety, reasoning and general applications: \texttt{PKU-SafeRLHF} \citep{ji2025pku}, a dataset of prompts paired with responses annotated for harmlessness; \texttt{GSM8K} \citep{cobbe2021gsm8k}, a collection of grade-school math word problems; and \texttt{UltraFeedback} \citep{cui2024ultrafeedback}, a broad-coverage preference dataset spanning instruction-following, helpfulness, and honesty. Since \texttt{GSM8K} is not natively a preference dataset, we construct preference pairs by pairing a randomly chosen correct response against the highest-scoring incorrect one for each prompt. We refer to Appendix \ref{app:additional-plots-gsm8k} for more details. For each task, we reshape the base distribution by either increasing or decreasing the probability of responses whose scores fall into a chosen percentile band, allowing us to concentrate or disperse mass in different parts of the quality spectrum. The midpoint of the percentile band is parameterized by $c$. We then generate Best-vs-Random preference data from the modified base distribution with $N = 4$, train a reward model using the same setup as in the previous experiments and compare test accuracy across different reshaped distributions. Full details of the upsampling/downsampling procedure and parameterization are given in Appendix \ref{app:distribution-reshaping}

\paragraph{Results.} Our goal is to understand whether reshaping the base distribution following our design principle can improve the accuracy. Figure \ref{fig:pku-saferlhf-combined}, \ref{fig:gsm8k-combined} and \ref{fig:ultrafeedback-combined} compare the test accuracy of reward models trained from BoN preference data with reshaped base distribution against the original base distribution.  Overall, the results support our prediction: increasing between mass on the comparisons that matter for a given task tends to improve accuracy.

 \textbf{Safety.} The clearest improvement comes from upsampling lower-score responses ($c = 15$) while upsampling higher-score responses ($c = 85$) substantially hurts the accuracy. Downsampling lower-score responses also hurts accuracy. The score distributions show the same pattern: the best-performing base distribution places more mass on lower-score responses, whereas the worst-performing one shifts mass toward higher-score responses. This supports our prediction that, for safety tasks, keeping more mediocre responses in the base distribution increases between mass over the comparisons that matter, namely harmful versus acceptable outputs.

\begin{figure}[h]
    \centering
    \begin{subfigure}[b]{0.60\textwidth} 
        \centering
        \includegraphics[width=\textwidth]{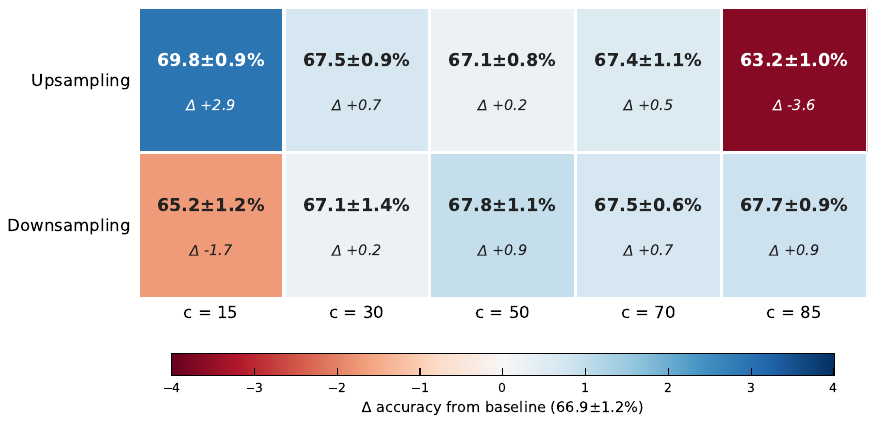}
        \caption{Test accuracy for different base distributions.}
        \label{fig:pku-saferlhf-pminus}
    \end{subfigure}
    \hfill
    \begin{subfigure}[b]{0.38\textwidth} 
        \centering
        \includegraphics[width=\textwidth]{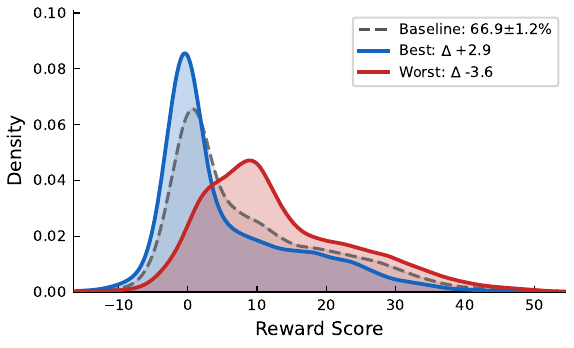}
        \caption{Distribution of the reward score between the best and worst performing  distribution.}
        \label{fig:pku-saferlhf-best-vs-worst}
    \end{subfigure}
    \caption{Heat map of test accuracy for different base distribution averaged over 10 random seeds, reported with a standard error(left) and the comparison between the distribution of reward scores between the best performing and worst performing base distribution (right) for \texttt{PKU-SafeRLHF}.}
    \label{fig:pku-saferlhf-combined}
\end{figure}

\textbf{Reasoning.}  Downsampling medium-score responses ($c = 50$) leads to a higher accuracy while upsampling high-score responses ($c = 85$) leads to a lower accuracy. This suggests that in our setup, the useful comparisons are not between strong and very strong responses, but between plausible but incorrect responses and the correct responses.  The original base distribution is already strong: generated responses are concentrated around reward score 25, while in the test distribution the rejected responses average around 15 and the chosen responses around 25 (Appendix \ref{app:additional-plots-gsm8k}). As a result, pushing more mass toward the very top end reduces the between mass over the comparisons that matter. This pattern is consistent with our theory.

\begin{figure}[h]
    \centering
    \begin{subfigure}[b]{0.60\textwidth} 
        \centering
        \includegraphics[width=\textwidth]{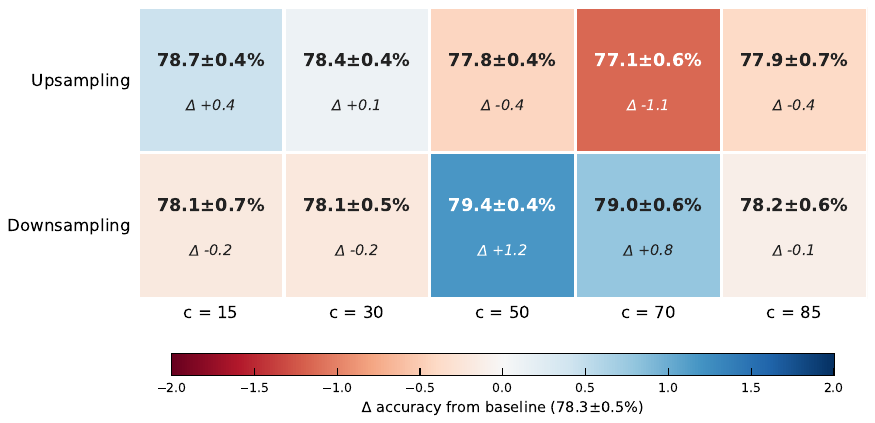}
        \caption{Test accuracy for different base distributions.}
        \label{fig:gsm8k-pminus}
    \end{subfigure}
    \hfill
    \begin{subfigure}[b]{0.38\textwidth} 
        \centering
        \includegraphics[width=\textwidth]{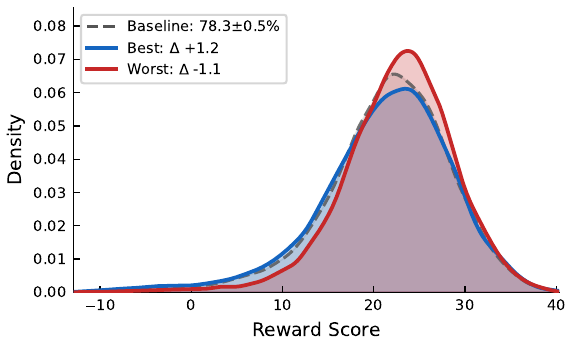}
        \caption{Distribution of the reward score between the best and worst performing  distribution.}
        \label{fig:gsm8k-best-vs-worst}
    \end{subfigure}
    \caption{Heat map of test accuracy for different base distribution averaged over 10 random seeds, reported with a standard error (left) and the comparison between the distribution of reward scores between the best performing and worst performing base distribution (right) for \texttt{GSM8K}.}
    \label{fig:gsm8k-combined}
\end{figure}

\textbf{General.} The key pattern is that upsampling mid-to-high score responses $(c = 30,50,85)$ lowers accuracy, while downsampling in the same region improves the accuracy. This suggests that for this task, it is better to spread probability mass across this part of the quality range rather than concentrate it. 
The distribution of the reward score provides a weaker but compatible evidence: the worst-performing base distribution is slightly more peaked than the best-performing one.

\begin{figure}[h]
    \centering
    \begin{subfigure}[b]{0.60\textwidth} 
        \centering
        \includegraphics[width=\textwidth]{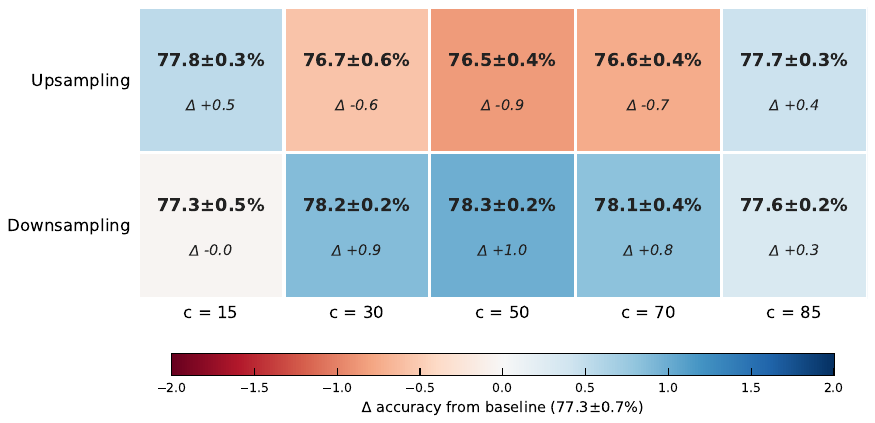}
        \caption{Test accuracy for different base distributions.}
        \label{fig:ultrafeedback-pminus}
    \end{subfigure}
    \hfill
    \begin{subfigure}[b]{0.38\textwidth} 
        \centering
        \includegraphics[width=\textwidth]{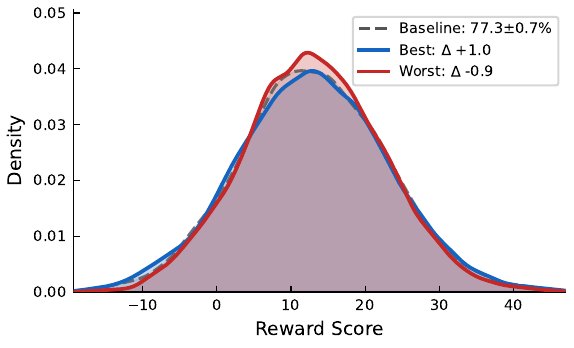}
        \caption{Distribution of the reward score between the best and worst performing  distribution.}
        \label{fig:ultrafeedback-best-vs-worst}
    \end{subfigure}
    \caption{Heat map of test accuracy for different base distribution averaged over 10 random seeds, reported with a standard error (left) and the comparison between the distribution of reward scores between the best performing and worst performing base distribution (right) for \texttt{UltraFeedback}.}
    \label{fig:ultrafeedback-combined}
\end{figure}

\section{Conclusion and Future Work}
We study preference data generated by Best-of-$N$ sampling through the lens of reward learning. This perspective lets us characterize the reward targets induced by the data, identify the trade-off between margin and connectivity, and derive practical principles for data design. In particular, our analysis suggests choosing $N$ based on the resource bottleneck, and shaping the base distribution to place more mass between the comparisons that matter under the target test distribution. Several directions remain open, including understanding the small-$N$ regime, going beyond Bradley--Terry reward learning, and determining how these data-design effects propagate to downstream policy optimization and iterative training. More broadly, our results show that preference-data construction is not merely a collection detail: it shapes both what reward is learned and how efficiently it can be learned.

\subsubsection*{Acknowledgments}
This work was supported in part by Bloomberg Data Science Fellowship, AFRL and DARPA via FA8750-23-2-1015, ONR via N00014-23-1-2368, and NSF via IIS-1955532 and IIS-1901403. The authors thank Korinna Fragkia for her insightful feedback on the draft.

\bibliography{reference.bib}
\bibliographystyle{tmlr}

\appendix
\section{Proofs for Characterizing the Learned Reward}
\label{app:characterizing-learned-reward}
In this section, we provide full proof details for results in Section \ref{sec:characterizing-learned-reward}.

\BoNDistribution*
\begin{proof}
For a specific response $y$ to be selected as the Best-of-N, two conditions must hold:
    \begin{enumerate}
        \item At least one of the $N$ samples, $y_1,\dots, y_N$ equals $y$
        \item No sample has score strictly greater than $r^*(x, y)$
    \end{enumerate}
    Since we assume that there is no tie in the reward score, the conditions are equivalent to when the highest reward score is $r^*(x,y)$ among $N$ samples. 
    \begin{align}
    \pBoN(y \mid x) &=  \Pr(\max_i r^*(x,y_i) = r^*(x,y))\\
    &=\Pr(r^*(x,y_i) \leq r^*(x,y)) - \Pr(r^*(x,y_i) < r^*(x,y)) \\
    &= F(x,y)^N - F^{-}(x,y)^N
    \end{align}
    The final line holds from independence between $y_1, \dots, y_N$.
    \end{proof}
This Lemma resembles Lemma 2.3 of \citet{beiramitheoretical} which was for the BoN policy. We can provide a similar result for the Worst-of-N distribution as well.
\WoNDistribution*
\begin{proof}
Similar to the Best-of-N, the condition at which a response $y$ is the Worst-of-N is when the lowest score is $r^*(x,y)$ among $N$ samples $y_1,\dots, y_N$.
\begin{align}
    \pWoN(y \mid x) &=  \Pr(\min_i r^*(x,y_i) = r^*(x,y))\\
    &=\Pr(r^*(x,y_i) \geq r^*(x,y)) - \Pr(r^*(x,y_i) > r^*(x,y)) \\
    &= (1 - F^{-}(x,y))^N - (1 - F(x,y))^N
    \end{align}
\end{proof}

We introduce the definition of Conditional Preference Distribution (CPRD) from \citet{pukdee2026doespreferencelearningrecover} which encodes the preference information of a comparison distribution. This allows us to refer to the result from \citet{pukdee2026doespreferencelearningrecover} and cleanly identify the population minimizer.

\begin{definition}[Conditional Preference Distribution (CPRD)]
    For a distribution $P$ over $\mathcal{X} \times \mathcal{Y} \times \mathcal{Y}$, the conditional preference distribution (CPRD) of $P$ is defined on any unordered pair $\{y,y'\}$ with $y \neq y'$ and $x \in \mathcal{X}$ as follows
    \begin{equation}
    \omega_P(y \succ y' \mid x) = \frac{P(x,y,y')}{P(x,y,y') + P(x,y',y)}.
    \end{equation}
\end{definition}

\Identificationrstarbi*

\begin{proof}
\textbf{Part 1: Show that $\rstarbi$ is a log ratio.} 

By Corollary 5.4 of \citet{pukdee2026doespreferencelearningrecover}, when the comparison distribution satisfies the positive--negative conditional independence, the population minimizer of the BT learning objective, recovers the CPRD of the comparison distribution.
\begin{equation}
P_{\rstarbi}(y \succ y' \mid x) = \omega_P(y \succ y' \mid x)
\end{equation}
almost everywhere with respect to $\PBI$ when $P_{\rstarbi}(y \succ y' \mid x) = \sigma(\rstarbi(x,y) - \rstarbi(x,y'))$ is the preference distribution induced by the score function $\rstarbi$. 

In addition, by Theorem 4.4 of \citet{pukdee2026doespreferencelearningrecover}, when the comparison distribution satisfies the positive--negative conditional independence, its CPRD is representable by a BT model where the score function is given by the log ratio,

\begin{equation}
    \omega_P(y \succ y' \mid x) = \sigma( \log \frac{\pBoN(y \mid x)}{\pbase} - \log \frac{\pBoN(y' \mid x)}{\pbaseyprime})
\end{equation}

Combining the two results, we can conclude that for all $x,y,y'$ with non-zero probability mass with respect to $\PBI$,
\begin{equation}
    \sigma(\rstarbi(x,y) - \rstarbi(x,y')) = \sigma( \log \frac{\pBoN(y \mid x)}{\pbase} - \log \frac{\pBoN(y' \mid x)}{\pbaseyprime})
\end{equation}

Since $\sigma$ is a strictly increasing function, we also have
\begin{equation}
    \rstarbi(x,y) - \rstarbi(x,y') = \log \frac{\pBoN(y \mid x)}{\pbase} - \log \frac{\pBoN(y' \mid x)}{\pbaseyprime}
\end{equation}

For any context $x$, let $y_0$ be a fixed outcome with a positive probability mass then for any $y$
\begin{equation}
    \rstarbi(x,y) = \log \frac{\pBoN(y \mid x)}{\pbase} + \rstarbi(x,y_0) - \log \frac{\pBoN(y_0 \mid x)}{\pbase}
\end{equation}

Set $c(x) = \rstarbi(x,y_0) - \log \frac{\pBoN(y_0 \mid x)}{\pbase}$ then we have
\begin{equation}
    \rstarbi(x,y) = \log \frac{\pBoN(y \mid x)}{\pbase} + c(x)
\end{equation}
for all $y$ with non-zero probability mass with respect to $\PBI$ as desired.

    \textbf{Part 2: Show that $\rstarbi$ preserves the ranking of $r^*$.} \\

    It suffices to show that $r^*(x, y) > r^*(x, y')$ implies
    \begin{equation}
    \label{eq:bon-ratio-inequality}
    \frac{\pBoN(y \mid x)}{P_\text{base}(y \mid x)} > \frac{\pBoN(y' \mid x)}{P_\text{base}(y' \mid x)}.
    \end{equation}

     By Lemma~\ref{lemma: bon-closed-form}, we have
    \begin{equation}
    \label{eq:bon-ratio-formula}
    \frac{\pBoN(y \mid x)}{P_\text{base}(y \mid x)} = \frac{F(x,y)^N - F^{-}(x,y)^N}{P_\text{base}(y \mid x)}.
    \end{equation}

    Using the factorization $a^N - b^N = (a - b)\sum_{j=0}^{N-1} a^{N-1-j}b^j$ and the fact that $F(x,y) - F^{-}(x,y) = P_\text{base}(y \mid x)$  this simplifies to
    \begin{equation}
    \label{eq:bon-ratio-simplified}
    \frac{\pBoN(y \mid x)}{P_\text{base}(y \mid x)} = \sum_{j=0}^{N-1} F(x,y)^{N-1-j}\, F^{-}(x,y)^j.
    \end{equation}

    Finally, by the definition of $F(x,y)$, we know that whenever $r^*(x,y) > r^*(x,y')$ we also have $F(x,y) > F(x, y')$ and $F^{-}(x,y) > F^{-}(x, y')$. Since  $\frac{\pBoN(y \mid x)}{P_\text{base}(y \mid x)}$ can be written as a polynomial in $F(x,y), F^-(x,y)$ with positive coefficients, we can conclude that 
    \begin{equation}
        \sum_{j=0}^{N-1} F(x,y)^{N-1-j}\, F^{-}(x,y)^j > \sum_{j=0}^{N-1} F(x,y')^{N-1-j}\, F^{-}(x,y')^j.
    \end{equation}
    \begin{equation}
        \frac{\pBoN(y \mid x)}{P_\text{base}(y \mid x)} > \frac{\pBoN(y' \mid x)}{P_\text{base}(y' \mid x)}
    \end{equation}
    This concludes our proof.
    \end{proof}

\Identificationrstarbwi*
\begin{proof}
The first part follows the same idea as in the proof of Theorem \ref{thm:identification-rstarbi} but just replace $P_\text{base}$ with $\pWoN$. For the second part, we need to show that $\rstarbwi$ preserves the ranking of $r^*$ where we will show that whenever $r^*(x,y) > r^*(x,y')$ we have 

\begin{equation}
    \frac{\pBoN(y \mid x)}{\pWoN(y \mid x)} > \frac{\pBoN(y' \mid x)}{\pWoN(y' \mid x)}.
\end{equation}

By Lemma~\ref{lemma: bon-closed-form} and Lemma~\ref{lemma: won-closed-form}, we have
\begin{equation}
    \frac{\pBoN(y \mid x)}{\pWoN(y \mid x)} = \frac{F(x,y)^N - F^{-}(x,y)^N}{(1 - F^{-}(x,y))^N - (1 - F(x,y))^N}
\end{equation}

Using the factorization $a^N - b^N = (a - b)\sum_{j=0}^{N-1} a^{N-1-j}b^j$ and the fact that $F(x,y) - F^{-}(x,y) = P_\text{base}(y \mid x)$  this simplifies to
\begin{equation}
    \frac{\pBoN(y \mid x)}{\pWoN(y \mid x)} = \frac{\sum_{j=0}^{N-1} F(x,y)^{N-1-j}\, F^{-}(x,y)^j}{\sum_{j=0}^{N-1} (1 - F^{-}(x,y))^{N-1-j}\, (1 - F(x,y))^j}.
\end{equation}
Here the term $(F(x,y) - F^{-}(x,y))$ cancels out.

By definition, $\rstar(x,y) > \rstar(x,y')$ implies $F(x,y) > F(x, y')$ and $1 - F^{-}(x,y) < 1 - F^{-}(x, y')$. Therefore, for the numerator, we have
\begin{equation}
    \sum_{j=0}^{N-1} F(x,y)^{N-1-j}\, F^{-}(x,y)^j > \sum_{j=0}^{N-1} F(x,y')^{N-1-j}\, F^{-}(x,y')^j.
\end{equation}

For the denominator, we have
\begin{equation}
    \sum_{j=0}^{N-1} (1 - F^{-}(x,y))^{N-1-j}\, (1 - F(x,y))^j < \sum_{j=0}^{N-1} (1 - F^{-}(x,y'))^{N-1-j}\, (1 - F(x,y'))^j.
\end{equation}

Putting everything together, we can conclude that
\begin{equation}
    \frac{\sum_{j=0}^{N-1} F(x,y)^{N-1-j}\, F^{-}(x,y)^j}{\sum_{j=0}^{N-1} (1 - F^{-}(x,y))^{N-1-j}\, (1 - F(x,y))^j} > \frac{\sum_{j=0}^{N-1} F(x,y')^{N-1-j}\, F^{-}(x,y')^j}{\sum_{j=0}^{N-1} (1 - F^{-}(x,y'))^{N-1-j}\, (1 - F(x,y'))^j}.
\end{equation}
Finally,
\begin{equation}
    \frac{\pBoN(y \mid x)}{\pWoN(y \mid x)} > \frac{\pBoN(y' \mid x)}{\pWoN(y' \mid x)}.
\end{equation}
This concludes our proof.
\end{proof}

\ObjectiveForBRandBW*

\begin{proof}
From Theorem 5.2 of \citet{pukdee2026doespreferencelearningrecover}, the BT learning objective can be decomposed as the cross entropy between the CPRD and the preference distribution induced by the score function.

\begin{align}
    \mathcal{L}_P(r) = C - Z \mathop{\mathbb{E}}_{(x, \{y,y'\}) \sim \widetilde P} \Big[ 
    \omega_P \log \sigma(\Delta_r) + (1 - \omega_P) \log (\sigma(-\Delta_r))
    \Big] \label{eq:decomposition-discriminative-BT-objective}
\end{align}
where $\omega_P(y \succ  y' \mid  x)$ is the CPRD of $P$ and $\Delta_r(x,y,y') = r(x,y) - r(x,y')$ is the margin between two outcomes. $C, Z$ are constants and $\widetilde P$ is a comparison distribution induced by $P$.

We claim that the CPRD of the Best-vs-Random scenario is either 0 or 1. For any context and pairs of outcomes $(x,\{y,y'\})$, let $y_{\max} \in \arg\max_{u\in\{y,y'\}} r^*(x,u)$ and $y_{\min} \in \arg\min_{u\in\{y,y'\}} r^*(x,u)$ be the outcome with a higher and lower reward score. By the definition of Best-vs-Random, the reward score of the positive sample is always higher than the reward score of the negative sample. Therefore, we would always have $y_{\max}$ as a positive sample and $y_{\min}$ as a negative sample. Therefore, the CPRD is given by

\begin{equation}
    \omega_{P_{\text{BR}}}(y_{\max} \succ  y_{\min} \mid  x) = \frac{P(x,y_{\max},y_{\min})}{P(x,y_{\max},y_{\min}) + P(x,y_{\min},y_{\max})} = 1
\end{equation}
and 
\begin{equation}
    \omega_{P_{\text{BR}}}(y_{\min} \succ  y_{\max} \mid  x) = \frac{P(x,y_{\min},y_{\max})}{P(x,y_{\min},y_{\max}) + P(x,y_{\max},y_{\min})} = 0.
\end{equation}

Substitute this back, we have
\begin{equation}
    \omega_P \log \sigma(\Delta_r) + (1 - \omega_P) \log (\sigma(-\Delta_r)) = \log \sigma(r(x,y_{\max}) - r(x,y_{\min})).
\end{equation}

With this result, we can conclude that the minimizer satisfies

\begin{equation}
    \rstarbr \in \arg\min_{r\in\cH} \bE_{(x, \{y,y'\}) \sim \widetilde P_{\text{BR}}} \Big[ -\log \sigma(r(x,y_{\max}) - r(x,y_{\min})) \Big].
\end{equation}

The result also holds for Best-vs-Worst, which also guarantees that the reward score of the positive sample is always higher than the reward score of the negative sample.
\end{proof}

In the next part, we will show that the $\rstarbr$ would converge to the $\rstarbi$ as $N$ grows. To show this, we will first show that $\PBR$ would be close to $\PBI$ as $N$ increases. The intuition is that as $N$ is large, the random sample from the rest of outcomes would be close to drawing an independent sample from the base distribution. We start with first deriving the probability mass of $\PBR$.

\begin{lemma}[Probability mass of $\PBR$]
Assume that there is no tie in reward score, for any context $x$ and outcome $y, y'$ there exists a constant $c> 0$ for which the probability mass of $\PBR$ is given by
\begin{equation}
    \PBR(x, y, y') = 
    \begin{cases}
        0 &\text{if } y = y' \text{ or } r^*(x,y') > r^*(x,y) \\
        \frac{cN}{N-1}P(x) \pbaseyprime \pBoNNminus (y \mid x) & \text{otherwise}
    \end{cases}
\end{equation}
where $\pBoNNminus$ is the probability mass of $\pBoN$ with $N-1$ draws.
\end{lemma}
\begin{proof}
The first case is quite simple. When we see a pair $y = y'$, we do not add it to the comparison distribution. Therefore, the probability mass is 0. In addition, if $r^*(x,y') > r^*(x,y)$, then it is not possible for $y$ to be preferred over $y'$ given the context $x$. Therefore, the probability mass is 0 as well. For the second case, assume that $y,y'$ are two outcomes with $r^*(x,y) > r^*(x,y')$ and $y \neq y'$. We can first factorize
\begin{equation}
    \PBR(x, y, y') = P(x) \PBR(y^+ = y, y^- = y' \mid x)
\end{equation}
where $y^+$ and $y^-$ are the positive and negative samples. To compute the second conditional probability term, we use the fact that for a given sample $y_1, \dots, y_N$, the probability of the positive sample being $y$ and the negative sample being $y'$ is given by
\begin{equation}
\PBR(y^+ = y, y^- = y' \mid x, y_1, \dots, y_N) = 1[\arg\max_{y_i} r^*(x,y_i) = y] \cdot \frac{|\{ i: y_i = y' \}|}{N - 1}.
\end{equation}
The first term is an indicator function which ensures that the sample with the maximum reward score is $y$ while the second term is the probability that the random sample that is drawn uniformly from the rest of outcomes is $y'$ which is basically the number of outcomes that is equal to $y'$ divided by the number of remaining outcomes, $N-1$. We can rewrite this as
\begin{equation}
    \PBR(y^+ = y, y^- = y' \mid x, y_1, \dots, y_N) = \frac{1}{N-1} 1[\arg\max_{y_i} r^*(x,y_i) = y] \cdot \sum_{i=1}^N 1[y_i = y'].
\end{equation}
Taking the expectation over all $y_1, \dots, y_N$, we have
\begin{align}
    \PBR(y^+ = y, y^- = y' \mid x) &= \bE_{y_{1:N}} [\PBR(y^+ = y, y^- = y' \mid x, y_1, \dots, y_N)] \\
&= \frac{1}{N-1} \sum_{i=1}^N \Pr_{y_{1:N}}(y_i = y', \argmax_{y_j} r^*(x,y_j) = y)\\
&= \frac{1}{N-1} \sum_{i=1}^N \Pr_{y_{1:N}}(y_i = y', \argmax_{y_j, j \neq i} r^*(x,y_j) = y)  &( \text{since } r^*(x,y) > r^*(x,y'))\\
&= \frac{1}{N-1} \sum_{i=1}^N \Pr_{y_i}(y_i = y')P_{y_j, j \neq i}(\argmax_{y_j} r^*(x,y_j) = y)  &( \text{by independence})\\
& = \frac{N}{N-1} P_{\text{base}}(y' \mid x) P_{y_{2:N}}(\argmax_{y_j} r^*(x,y_j) = y)\\
& =  \frac{N}{N-1} \pbaseyprime \pBoNNminus (y \mid x)
\end{align}

Finally, we note that the above result only holds when $y \neq y'$. When $y = y'$, we have to replace $\frac{|\{ i: y_i = y' \}|}{N - 1}$ with $\frac{|\{ i: y_i = y' \}| - 1}{N - 1}$ since one sample is already chosen as $y$. Regardless, we don't need to worry about such case since we will discard the pair with $y = y'$ and set the probability mass to 0. To ensure that the probability mass is valid, we need to scale up the probability mass of the rest by some factor $c$. This concludes the proof.

\end{proof}

Before we prove that the total variation between $\PBR$ and $\PBI$ is small, we first provide a lemma which relates the total variation between two distributions with its unnormalized probability mass.

\begin{lemma}[Normalization Preserves Closeness]
\label{lemma:normalization-preserves-closeness}
Let $\Wp$ and $\Wq$ be two unnormalized probability mass functions over a finite set of outcomes $\Omega$. Let $Z_p: = \sum_{x \in \Omega} \Wp(x)$ and $Z_q: = \sum_{x \in \Omega} \Wq(x)$ be the normalization constants of $\Wp$ and $\Wq$ respectively with $P(x) = \frac{\Wp(x)}{Z_p}$ and $Q(x) = \frac{\Wq(x)}{Z_q}$ as the normalized probability mass functions. If there exists a constant $c > 0$ such that $Z_p \geq c$ and $Z_q \geq c$, then the total variation of $P,Q$ is bounded by the total variation of $\Wp, \Wq$,
\begin{equation}
    \TV(P,Q) \leq \frac{2}{c} \TV(\Wp, \Wq)
\end{equation}
\end{lemma}

\begin{proof}
Without loss of generality, assume that $Z_q \geq Z_p$. The total variation of $P,Q$ is given by
\begin{align}
\TV(P,Q) &= \frac{1}{2}\sum_{x \in \Omega} |P(x) - Q(x)|\\
&= \frac{1}{2}\sum_{x \in \Omega} | \frac{\Wp(x)}{Z_p} - \frac{\Wq(x)}{Z_q} |\\
&\leq \frac{1}{2}\sum_{x \in \Omega} | \frac{\Wp(x)}{Z_p} - \frac{\Wq(x)}{Z_p} | + \frac{1}{2}\sum_{x \in \Omega} |\frac{\Wq(x)}{Z_q} - \frac{\Wq(x)}{Z_p} |\\
&= \frac{1}{Z_p}\TV(\Wp, \Wq) + \frac{|Z_p - Z_q|}{2Z_p}
\end{align}
Since
\begin{equation}
\frac{1}{2}|Z_p - Z_q| = \frac{1}{2}|\sum_{x \in \Omega} \Wp(x) - \sum_{x \in \Omega} \Wq(x)| \leq \frac{1}{2}\sum_{x \in \Omega} |\Wp(x) - \Wq(x)| = \TV(\Wp, \Wq)
\end{equation}
We have
\begin{equation}
    \TV(P,Q) \leq \frac{2}{Z_p}\TV(\Wp, \Wq) \leq \frac{2}{c}\TV(\Wp, \Wq)
\end{equation}
\end{proof}

\begin{proposition}[Total Variation between $\PBR$ and $\PBI$]
    \label{prop:total-variation-br-bi}
    Assume that there is no tie in the reward score and we discard the pair with $y = y'$. Then, for any base distribution $P_{\text{base}}$, there exists a constant $c > 0$ for which for any $N \geq 2$, the total variation between $\PBR$ and $\PBI$ is given by
    \begin{equation}
    \TV(\PBR, \PBI) \leq \frac{c}{N-1}
    \end{equation}
\end{proposition}

\begin{proof}
For any fixed context $x$ and the positive outcome $y$, the probability mass of $\PBI$ is given by
\begin{equation}
\PBI(y, y' \mid x) = 
\begin{cases}
0 & \text{if } y = y' \\
\frac{1}{Z_I}\pBoN(y \mid x) P_{\text{base}}(y' \mid x) & \text{otherwise}
\end{cases}
\end{equation}
where $Z_I$ is a constant that compensates for the case when we discard the pair with $y = y'$.
On the other hand, the probability mass of $\PBR$ is given by
\begin{equation}
    \PBR(y, y' \mid x) = 
    \begin{cases}
        0 &\text{if } y = y' \\
        0 &\text{if } r^*(x,y') > r^*(x,y) \\
        \frac{1}{Z_R}\frac{N}{N-1}\pbaseyprime \pBoNNminus (y \mid x) & \text{otherwise}
    \end{cases}
\end{equation}
and $Z_R$ is also a normalization constant. To show that the total variation between $\PBR$ and $\PBI$ is small, from Lemma \ref{lemma:normalization-preserves-closeness} it is sufficient to show that the unnormalized probability mass of these two distributions are close. Define the unnormalized probability mass of $\PBR$ and $\PBI$ as

\begin{equation}
    \PBIunnorm(y, y' \mid x) = 
    \begin{cases}
    0 & \text{if } y = y' \\
    \pBoN(y \mid x) P_{\text{base}}(y' \mid x) & \text{otherwise}
    \end{cases}
    \end{equation}
\begin{equation}
    \PBRunnorm(y, y' \mid x) = 
    \begin{cases}
        0 &\text{if } y = y' \\
        0 &\text{if } r^*(x,y') > r^*(x,y) \\
        \frac{N}{N-1}\pbaseyprime \pBoNNminus (y \mid x) & \text{otherwise}
    \end{cases}
\end{equation}

Without loss of generality, let the set of outcomes $\{y_1, \dots, y_m\}$ be ordered such that $r^*(x,y_1) < r^*(x,y_2) < \dots < r^*(x,y_m)$. As a result, we can write
\begin{equation}
F(x,y_j) = F(x,y_{j-1}) + P_{\text{base}}(y_j \mid x)
\end{equation}
for any $j \in \{1, 2, \dots, m\}$ when $F(x,y_0) = 0$. Further, we will write $F_j = F(x,y_j)$ and $p_j = P_{\text{base}}(y_j \mid x)$ for the convenience of notation. With this notation, we know that $\PBRunnorm(y_i,y_j \mid x) = 0$ for any $i < j$ since the reward score of $y_i$ is lower than $y_j$ so it's not possible for $y_i$ to be preferred over $y_j$ under Best-vs-Random. Therefore, the total variation between $\PBIunnorm$ and $\PBRunnorm$ is given by

\begin{align}
2\TV(\PBIunnorm, \PBRunnorm) &= \sum_{y,y'} |\PBIunnorm(y, y' \mid x) - \PBRunnorm(y, y' \mid x)|\\
&= \sum_{i > j} |\PBIunnorm(y_i, y_j \mid x) - \PBRunnorm(y_i, y_j \mid x)| + \sum_{i < j} |\PBIunnorm(y_i, y_j \mid x) - \PBRunnorm(y_i, y_j \mid x)|\\
&= \sum_{i > j} |\PBIunnorm(y_i, y_j \mid x) - \PBRunnorm(y_i, y_j \mid x)| + \sum_{i < j} |\PBIunnorm(y_i, y_j \mid x) |
\end{align}

We can split the total variation into two parts. The first part is the difference between the probability mass of $\PBIunnorm$ and $\PBRunnorm$ when $i > j$ while the second part is just the probability mass of $\PBIunnorm$ when $i < j$. We bound the first part as follows:

\begin{align}
    \sum_{i > j} |\PBIunnorm(y_i, y_j \mid x) - \PBRunnorm(y_i, y_j \mid x)| &= \sum_{i > j} |\pBoN(y_i \mid x) P_{\text{base}}(y_j \mid x)  - \frac{N}{N-1}\pBoNNminus (y_i \mid x) P_{\text{base}}(y_j \mid x)|\\
    &= \sum_{i > j} | (F_i^N - F_{i-1}^N) - \frac{N}{N-1}(F_i^{N-1} - F_{i-1}^{N-1})| p_j
\end{align}

Let $g(u) = \frac{N}{N-1}u^{N-1} - u^N$.  Differentiating $g(u)$ with respect to $u$, we have $g'(u) = N u^{N-2}(1-u) \geq 0$ for all $u \in [0,1]$ which implies that $g(u)$ is increasing on $[0,1]$. We can write the term above in terms of $g$,
\begin{align}
     (F_i^N - F_{i-1}^N) - \frac{N}{N-1}(F_i^{N-1} - F_{i-1}^{N-1}) = g(F_{i-1}) - g(F_i) \leq 0
\end{align}
This follows from the fact that $g(u)$ is increasing while $F_i > F_{i-1}$. Therefore, we have
\begin{align}
    \sum_{i > j} |\PBIunnorm(y_i, y_j \mid x) - \PBRunnorm(y_i, y_j \mid x)| &= \sum_{i > j} (g(F_i) - g(F_{i-1})) p_j\\
    &= \sum_{j = 1}^{m-1}  p_j \sum_{i = j+1}^m (g(F_i) - g(F_{i-1})) \\
    &= \sum_{j = 1}^{m-1}  p_j (g(F_{m}) - g(F_j)) \\
   &\leq \sum_{j = 1}^{m-1}  p_j g(F_{m}) \\
   &\leq g(F_m) = g(1) = \frac{1}{N-1}
\end{align}

For the second part, we have
\begin{align}
    \sum_{i < j} |\PBIunnorm(y_i, y_j \mid x)|  &= \sum_{i < j} (F_i^N - F_{i-1}^N) p_j \\
    &= \sum_{i = 1}^{m-1} \sum_{j = i+ 1}^{m} (F_i^N - F_{i-1}^N) p_j \\
    &= \sum_{i = 1}^m (F_i^N - F_{i-1}^N)  \sum_{j = i+ 1}^{m} p_j \\
    &= \sum_{i = 1}^{m-1} (F_i^N - F_{i-1}^N) (1 - F_i) \\
    &= \sum_{i = 1}^{m-1} F_i^N - F_{i-1}^N - F_i^{N+1} + F_iF_{i-1}^N\\
    &= F_{m-1}^N -  \sum_{i=1}^{m-1} F_i^{N+1} + \sum_{i=1}^{m-1} (F_{i-1} + p_i)F_{i-1}^N\\
    &= F_{m-1}^N -  \sum_{i=1}^{m-1} F_i^{N+1} + \sum_{i=0}^{m-2} F_i^{N + 1} + \sum_{i=1}^{m-1} p_iF_{i-1}^N \\
    &= F_{m-1}^N - F_{m-1}^{N + 1} + \sum_{i=1}^{m-1} p_iF_{i-1}^N\\
    &= F_{m-1}^N (1 - F_{m-1}) + \sum_{i=1}^{m-1} p_iF_{i-1}^N\\
    &= \sum_{i=1}^m p_iF_{i-1}^N
\end{align}

Here we use the telescoping technique. We observe that the third term is a left Riemann sum for approximating the integral of $u^N$ between $0$ and $1$ by splitting $[0,1]$ into intervals with width $p_1, p_2, \dots, p_{m-1}, p_m$. Since this function is increasing, we can conclude that the Riemann sum is upper bounded by the true integral.
\begin{equation}
    \sum_{i=1}^m F_{i-1}^N p_i \leq \int_0^1 u^N du = \frac{1}{N+1}
\end{equation}

Therefore,
\begin{equation}
    \sum_{i < j} |\PBIunnorm(y_i, y_j \mid x)|  \leq \frac{1}{N+1}
\end{equation}

Combining two parts, we can conclude that

\begin{equation}
    2\TV(\PBIunnorm, \PBRunnorm) \leq \frac{1}{N-1} + \frac{1}{N+1} \leq \frac{2}{N-1}
\end{equation}

The final step we need is to show that the normalization constant of $\PBI$ and $\PBR$ are bounded below by a constant. We start with $Z_I$.

\begin{align}
    Z_I &= \sum_{i \neq j} \PBIunnorm(y_i, y_j \mid x)\\
&= 1 - \sum_{i = 1}^m \PBIunnorm(y_i, y_i \mid x)\\
&= 1 - \sum_{i = 1}^m (F_i^N - F_{i-1}^N) p_i\\
&\geq 1 - (\max p_i) \sum_{i = 1}^m (F_i^N - F_{i-1}^N) = 1 - \max p_i
\end{align}

$Z_I$ is bounded below by $1 - \max p_i$ for any $N$. In addition, for $Z_R$, we have

\begin{align}
Z_R &= \sum_{i > j} \PBRunnorm(y_i, y_j \mid x)\\
&= \sum_{i > j} \frac{N}{N-1} (F_i^{N-1} - F_{i-1}^{N-1}) p_j\\
& \geq \frac{N}{N-1} (F_m^{N-1} - F_{m-1}^{N-1}) p_m\\
&= \frac{N}{N-1} (1 - (1 - p_m)^{N-1})p_m\\
&\geq p_m^2
\end{align}
The lower bound here is fairly loose but it is sufficient for us to use Lemma \ref{lemma:normalization-preserves-closeness} which now implies that there exists a constant $c_x > 0$ such that
\begin{equation}
\TV(\PBR(y,y' \mid x), \PBI(y,y' \mid x)) \leq \frac{c_x}{N-1}
\end{equation}
for any $x \in \mathcal{X}$. Now, we take the summation over all $x$ and will have the total variation between $\PBR$ and $\PBI$ is bounded by
\begin{equation}
\TV(\PBR, \PBI) =  \sum_{x} P(x) \TV(\PBR(y,y' \mid x), \PBI(y,y' \mid x)) \leq \frac{c}{N-1}
\end{equation}
when $c = \sum_{x} P(x) c_x$. This concludes the proof.

\end{proof}

Now, we have established that $\PBR$ and $\PBI$ are close as $N$ grows, with this we can also show that the population risk minimizer would also be close to each other as well. Here we apply a standard technique in learning theory that first shows that a small TV distance means that the difference in the population risk would be small. Then, we use the strong convexity of the population risk to show that the difference in the population risk minimizer would be small as well. We will first prove some technical lemmas that will be useful for the proof and then will present the main proof accordingly.

The first lemma shows that we can bound the population risk difference in terms of the total variation between the two distributions.
\begin{lemma}
\label{lemma: small tv leads to small population risk difference}
Let $\cH$ be a hypothesis class that is bounded by some constant $B$. Let $P$ and $Q$ be two comparison distributions then there exists a constant $M_B$ such that
\begin{equation}
    \sup_{r \in \cH} |\LP(r) - \LQ(r)| \leq 2M_B \TV(P,Q)
\end{equation}
\end{lemma}

\begin{proof}
First, we will show that if $\cH$ is bounded by $B$ then the BT loss is also bounded as well. Recall that the loss on any $(x, y, y')$ is given by
\begin{equation}
    \ell(\Deltar) = -\log \sigma(r(x,y^+) - r(x,y^-)) = \log (1 +  \exp(r(x, y^-) - r(x, y^+))).
\end{equation}
Since $r$ is bounded by $B$, we have $|r(x, y^-) - r(x, y^+)| \leq 2B$ so that
\begin{equation}
    \ell(\Deltar) \leq \log (1 +  \exp(2B)): = M_B.
\end{equation}

Next, we will show that the difference in the population risk is bounded by the total variation between the two distributions.
\begin{align}
    |\LP(r) - \LQ(r)| &= |\sum_{x,y,y'} P(x,y,y') \ell(\Deltar(x,y,y')) - \sum_{x,y,y'} Q(x,y,y') \ell(\Deltar(x,y,y'))| \\
    & = |\sum_{x,y,y'} (P(x,y,y') - Q(x,y,y')) \ell(\Deltar(x,y,y'))| \\
    & \leq \sum_{x,y,y'} |P(x,y,y') - Q(x,y,y')| \ell(\Deltar(x,y,y'))| \\
    & \leq \sum_{x,y,y'} |P(x,y,y') - Q(x,y,y')| M_B \\
    & = 2M_B \TV(P,Q)
\end{align}
Since this is true for any $r \in \cH$, it would also hold for the supremum over all $r \in \cH$. This concludes the proof.
\end{proof}

The second lemma shows that if the population risk between two distributions is small on all hypotheses in the hypothesis class, then the population risk minimizer of one distribution would also perform well on the other distribution.
\begin{lemma}
    \label{lemma: small population risk difference means minimizer is close}
Let $\cH$ be a hypothesis class and $P$, $Q$ be two comparison distributions such that $\sup_{r \in \cH} |\LP(r) - \LQ(r)| \leq \epsilon$. Let $\rstarp$ and $\rstarq$ be the population risk minimizer of $\LP$ and $\LQ$ respectively. Then, we have
\begin{align}
\LP(\rstarq) \leq \LP(\rstarp) + 2\epsilon \\
\LQ(\rstarp) \leq \LQ(\rstarq) + 2\epsilon 
\end{align}
\end{lemma}

\begin{proof}
We write the difference between the population risk of $\rstarp$ and $\rstarq$ as
\begin{align}
    \LP(\rstarq) - \LP(\rstarp) &=  (\LP(\rstarq) - \LQ(\rstarq)) + (\LQ(\rstarq) - \LQ(\rstarp)) + (\LQ(\rstarp) - \LP(\rstarp))\\
    &\leq |\LP(\rstarq) - \LQ(\rstarq)| + (\LQ(\rstarq) - \LQ(\rstarp)) + |\LQ(\rstarp) - \LP(\rstarp)|\\
    &\leq \epsilon + 0 + \epsilon = 2\epsilon
\end{align}
The first and third term follows from the assumption that $\sup_{r \in \cH} |\LP(r) - \LQ(r)| \leq \epsilon$. The second term is at most 0 since $\rstarq$ is the population risk minimizer of $\LQ$ so that $\LQ(\rstarq) \leq \LQ(r)$ for any $r \in \cH$ including  $\rstarp$. We have the first result and the second result follows from the same argument.
\end{proof}

In this Lemma, we show that when the hypothesis class is convex, we can bound the difference between the margin of a reward function with the population risk minimizer with the difference in their population risk.

\begin{lemma}
\label{lemma: margin semi-norm is bounded by population risk difference}
Let $\cH$ be a convex hypothesis class that is bounded by some constant $B$. Let $P$ be a comparison distribution and $\rstarp$ be the population risk minimizer of $\LP$. There exists a constant $c_B$ such that for any reward function $r \in \cH$, we have
\begin{equation}
    \lVert r - \rstarp \rVert_{\Delta, P} \leq \frac{2}{\sigma_B} (\LP(r) - \LP(\rstarp))
\end{equation}
\end{lemma}

\begin{proof}
First, we will show that when $\cH$ is bounded, $\LP(r)$ is strongly convex. This result follows from the proof of Theorem 6.7 in \citet{pukdee2026doespreferencelearningrecover}. We provide a brief proof here for completeness. We can write the BT objective as
\begin{equation}
    \LP(r) = C + \mathbb{E}_{(x, \{y,y'\} ) \sim \widetilde{P}} \left[ g_{x, {y,y'}}(\Deltar(x,y,y')) \right].
\end{equation}
where  $g_{x, {y,y'}}: \mathbb{R} \to \mathbb{R}$, is defined for each $(x, \{y,y'\})$ as 
\begin{align}
    g_{x, {y,y'}}(a) 
    &= - \omega_P(y \succ y' \mid x) \log \sigma(a) - (1- \omega_P(y \succ y' \mid x)) \log (1- \sigma(a))
\end{align}
We will write $\Deltar$ for $\Deltar(x,y,y')$ for convenience from now. The first thing we will show is that $g_{x, {y,y'}}$ is strongly convex with respect to $a$. We can check this by taking the derivative.
Using the property that $\sigma'(a) = \sigma(a)(1- \sigma(a))$, we have the first derivative
\begin{equation}
    g_{x, {y,y'}}'(a) = \sigma(a) - \omega_P
\end{equation}
and the second derivative
\begin{equation}
    g_{x, {y,y'}}''(a) = \sigma(a)(1- \sigma(a))
\end{equation}
Since the score function $r$ is bounded by a constant $B$, for any $(x, \{y,y'\})$, we know that $|\Deltar| = |r(x,y) - r(x,y')| \leq 2B$. 
Now, for $a \in [-2B, 2B]$  we know that 
\begin{equation}
\sigma(a)(1- \sigma(a)) \geq \sigma(2B)(1- \sigma(2B)) := \sigma_B > 0
\end{equation}
With this assumption, we can conclude that $g_{x, {y,y'}}$ is $\sigma_B$-strongly convex. Recall that if $f$ is $\sigma_B$-strongly convex, then for any $x,y$, we have
\begin{equation}
    f(x) - f(y) \geq f'(y)(x-y) + \frac{\sigma_B}{2}(x-y)^2.
\end{equation}
Apply this inequality on $g_{x, {y,y'}}$ with $\Deltar$ and $\Deltarstarp$ and drop $(x, y, y')$ for convenience.
\begin{equation}
    g_{x, {y,y'}}(\Deltar) - g_{x, {y,y'}}(\Deltarstarp) \geq g'_{x, {y,y'}}(\Deltarstarp)(\Deltar - \Deltarstarp) + \frac{\sigma_B}{2}(\Deltar - \Deltarstarp)^2.
\end{equation}
Take the expectation over the distribution $\widetilde{P}$, we have
\begin{equation}
    \mathbb{E}_{\widetilde{P}} [g_{x, {y,y'}}(\Deltar) - g_{x, {y,y'}}(\Deltarstarp)] \geq \mathbb{E}_{\widetilde{P}} [g'_{x, {y,y'}}(\Deltarstarp)(\Deltar - \Deltarstarp)] + \frac{\sigma_B}{2} \mathbb{E}_{\widetilde{P}} [(\Deltar - \Deltarstarp)^2].
\end{equation}
The left hand side is just the difference between the population risk, while the second term on the right hand side is the margin semi-norm with respect to $P$; since $(\Deltar - \Deltarstarp)^2$ is invariant under swapping $y$ and $y'$, its expectation under $\widetilde P$ is equal to the expectation over $P$ which is exactly $\lVert r - \rstarp \rVert_{\Delta, P}$.

\begin{equation}
    \LP(r) - \LP(\rstarp) \geq \mathbb{E}_{\widetilde{P}} [g'_{x, {y,y'}}(\Deltarstarp)(\Deltar - \Deltarstarp)] + \frac{\sigma_B}{2} \lVert r - \rstarp \rVert_{\Delta, P}.
\end{equation}

Now, the final part is to show that the first term on the right hand side is non-negative. We achieve this by using the property of $\cH$ being convex. From convexity, we know that for any $t \in [0,1]$, we have
\begin{equation}
r_t = (1 - t) \rstarp + t r \in \cH
\end{equation}
Since $\rstarp$ is the population risk minimizer, the right derivative of the population risk with respect to $t$ at $t = 0$ must be non-negative.
\begin{equation}
\frac{\partial }{\partial t} \LP(r_t) \bigg|_{t = 0} \geq 0
\end{equation}
where
\begin{align}
    \frac{\partial }{\partial t} \LP(r_t) &= \frac{\partial }{\partial t} \bE_{\widetilde{P}} [g_{x, {y,y'}}(\Delta_{r_t})]\\
    &= \bE_{\widetilde{P}} [g'_{x, {y,y'}}(\Delta_{r_t}) \frac{\partial }{\partial t} \Delta_{r_t}]
\end{align}
We can derive 
\begin{align}
    \frac{\partial }{\partial t} \Delta_{r_t} &= \frac{\partial }{\partial t} ((1 - t) \Deltarstarp + t \Deltar) =  \Deltar - \Deltarstarp.
\end{align}
Therefore,
\begin{equation}
    \frac{\partial }{\partial t} \LP(r_t) = \bE_{\widetilde{P}} [g'_{x, {y,y'}}(\Delta_{r_t}) (\Deltar - \Deltarstarp)]
\end{equation}
Setting $t = 0$ and substituting back, we would have that
\begin{equation}
    \bE_{\widetilde{P}} [g'_{x, {y,y'}}(\Deltarstarp) (\Deltar - \Deltarstarp)] \geq 0
\end{equation}
as desired. As a result, we can conclude that
\begin{equation}
    \LP(r) - \LP(\rstarp) \geq \frac{\sigma_B}{2} \lVert r - \rstarp \rVert_{\Delta, P}.
\end{equation}
This concludes the proof.
\end{proof}

Now, we are ready to prove the Theorem \ref{thm:quality-rstarbr}.
\Qualityofrstarbr*

\begin{proof}
First, by proposition \ref{prop:total-variation-br-bi}, we have that the total variation between $\PBR$ and $\PBI$ is bounded by $\frac{c}{N-1}$ for some constant $c > 0$,
\begin{equation}
    \TV(\PBR, \PBI) \leq \frac{c}{N-1}.
\end{equation}
By Lemma \ref{lemma: small tv leads to small population risk difference}, we can bound the population risk difference with the total variation,
\begin{equation}
\sup_{r \in \cH} |\LBR(r) - \LBI(r)| \leq 2M_B \TV(\PBR, \PBI) \leq \frac{2M_B c}{N-1}.
\end{equation}
By Lemma \ref{lemma: small population risk difference means minimizer is close}, we would have
\begin{equation}
\LBR(\rstarbi) - \LBR(\rstarbr)\leq  \frac{4M_B c}{N-1}
\end{equation}
Finally, by the convexity of $\cH$ and Lemma \ref{lemma: margin semi-norm is bounded by population risk difference}, we have
\begin{equation}
\lVert \rstarbr - \rstarbi \rVert_{\Delta, \PBR} \leq \frac{2}{\sigma_B} (\LBR(\rstarbi) - \LBR(\rstarbr)) \leq  \frac{8M_B c}{\sigma_B (N-1)}.
\end{equation}
Setting $c_B = \frac{8M_B c}{\sigma_B}$ concludes the proof.
\end{proof}

The proof for Theorem \ref{thm:quality-rstarbw} follows the same argument. The key difference is to show that the total variation between $\PBW$ and $\PBWI$ now decays exponentially.

\begin{proposition}[Total Variation between $\PBW$ and $\PBWI$]
    \label{prop:total-variation-bw-bwi}
    Assume that there is no tie in the reward score and we discard the pair with $y = y'$. Then, there exists a constant $c > 0$ and $\rho < 1$ for which for any $N \geq 2$, the total variation between $\PBW$ and $\PBWI$ is given by
    \begin{equation}
    \TV(\PBW, \PBWI) \leq c\rho^N
    \end{equation}
\end{proposition}

\begin{proof}
    The main idea of the proof is to bound the total variation between $\PBW$ and $\PBWI$ on the pair of best and worst outcome. This is because we know that as $N$ grows, we are more likely to see the outcome with the highest and lowest true reward score. For the first step, we will calculate the probability mass of $\PBW$ and $\PBWI$.\\
    
    \textbf{Part 1: Calculate the probability mass of $\PBWI$ and $\PBW$.}
    We start with the probability mass of $\PBWI$.
    \begin{equation}
        \PBWI(y, y' \mid x) = 
        \begin{cases}
        0 & \text{if } y = y' \\
        \frac{1}{c_I}\pBoN(y \mid x) \pWoN(y' \mid x) & \text{otherwise}
        \end{cases}
        \end{equation}
    Here $c_I$ is a constant that ensures that $\PBWI$ is a valid probability distribution. We need this constant because whenever $y = y'$, we discard the triplet. We will bound this constant term.  Without loss of generality, assume that the set of outcomes is given by $y_1, y_2, \dots, y_m$ where $r^*(x,y_1) < r^*(x,y_2) < \dots < r^*(x,y_m)$ so that $y_1$ is the worst outcome and $y_m$ is the best outcome. 
    We also define the unnormalized probability mass of $\PBWI$ as
    \begin{equation}
        \PBWIunnorm(y, y' \mid x) = \pBoN(y \mid x) \pWoN(y' \mid x)
    \end{equation}
    From Lemma \ref{lemma: bon-closed-form}  and \ref{lemma: won-closed-form}, we have
    \begin{align}
        \PBWIunnorm(y_i, y_j \mid x) &= \pBoN(y_i \mid x) \pWoN(y_j \mid x)\\
        &= (F(x,y_i)^N - F^{-}(x,y_i)^N) ((1 - F^{-}(x,y_j))^N - (1 - F(x,y_j))^N)\\
        &= (F(x,y_i)^N - F(x,y_{i-1})^N) ((1 - F(x,y_{j-1}))^N - (1 - F(x,y_j))^N)\\
        &= (F_i^N - F_{i-1}^N) ((1 - F_{j-1})^N - (1 - F_j)^N)
    \end{align}
    This follows from the fact that $F^{-}(x,y_i) = F(x,y_{i-1})$. We also write $F_i = F(x,y_i)$ for convenience. It is clear that $\PBWI$ is just $\PBWIunnorm$ normalized by some constant that takes care of the case when we discard the pair with $y = y'$. Therefore,
    \begin{align}
        c_I &= \sum_{y \neq y'} \PBWIunnorm(y, y' \mid x)\\
        &= 1 - \sum_{i=1}^m \PBWIunnorm(y_i, y_i \mid x)\\
        &= 1 - \sum_{i=1}^m (F_i^N - F_{i-1}^N) ((1 - F_{i-1})^N - (1 - F_i)^N)
    \end{align}
    For each $i$ between 1 and $m$, the term $(F_i^N - F_{i-1}^N) ((1 - F_{i-1})^N - (1 - F_i)^N)$ would converge to $0$ as $N$ grows. Further, when $i = 1$, we have
    \begin{equation}
        (F_1^N - F_0^N)((1 - F_0)^N - (1 - F_1)^N) = F_1^N ( 1 - (1 - F_1)^N)
    \end{equation}
    and when $i = m$, we have
    \begin{equation}
    (F_m^N - F_{m-1}^N)((1 - F_{m-1})^N - (1 - F_m)^N) = (1 - F_{m-1}^N)(1 - F_{m-1})^N
    \end{equation}
    Both terms also converge to $0$ as $N$ grows. As a result, we can conclude that $c_I \to 1$ as $N$ grows and there exists a constant $a > 0$ for which $c_I \geq a > 0$ for all $N$. Combining everything together, we have
    \begin{equation}
        \PBWI(y_i, y_j \mid x) = 
        \begin{cases}
        0 & \text{if } i = j \\
        \frac{1}{c_I}(F_i^N - F_{i-1}^N) ((1 - F_{j-1})^N - (1 - F_j)^N) & \text{otherwise}
        \end{cases}
        \end{equation}
    when $c_I \to 1$ as $N$ grows. Next, we will calculate the probability mass of $\PBW$. First, for any $i \leq j$ we know that
    \begin{equation}
        \PBW(y_i, y_j \mid x) = 0
    \end{equation}
because $y_j$ has a higher reward score than $y_i$ so it's impossible for $y_i$ to be preferred over $y_j$. For the case when $i > j$, the probability is not zero anymore. Let $Y_1, Y_2, \dots, Y_N$ be the $N$ outcomes sampled from the base distribution and let $\PBWunnorm$ be the unnormalized probability mass of $\PBW$ then
\begin{align}
    \PBWunnorm(y_i, y_j \mid x)&= \Pr(\arg\max_{Y_i} r^*(x,Y_i) = y_i, \arg\min_{Y_i} r^*(x,Y_i) = y_j)\\
    &= \Pr(\max r^*(x, Y_i) = r^*(x, y_i), \min r^*(x, Y_i) = r^*(x, y_j))\\
    &= \Pr( r^*(x,y_j) \leq r^*(x,Y_i) \leq r^*(x, y_i)) -  \Pr( r^*(x,y_j) \leq r^*(x,Y_i) < r^*(x, y_i)) \\
    &\quad - \Pr( r^*(x,y_j) < r^*(x,Y_i) \leq r^*(x, y_i)) + \Pr( r^*(x,y_j) < r^*(x,Y_i) < r^*(x, y_i))\\
    &= (F_i - F_{j-1})^N - (F_{i-1} - F_{j-1})^N - (F_{i} - F_j)^N + (F_{i-1} - F_j)^N
\end{align}
This follows from the assumption that there is no tie in the reward score and the second to last line holds from an inclusion-exclusion principle. We have
\begin{equation}
    \PBW(y_i, y_j \mid x) = 
    \begin{cases}
    0 & \text{if } i \leq j \\
   \frac{1}{c_D} \PBWunnorm(y_i, y_j \mid x)
    \end{cases}
    \end{equation}
Similarly $c_D$ is a constant that ensures that $\PBW$ is a valid probability distribution which handles the case when $i \leq j$. We know that
\begin{equation}
    c_D = \sum_{i > j} \PBWunnorm(y_i, y_j \mid x) = \sum_{j = 1}^{m-1}\sum_{i = j+1}^m \PBWunnorm(y_i, y_j \mid x)
\end{equation}
Let $S_j = \sum_{i = j+1}^m \PBWunnorm(y_i, y_j \mid x)$ then we can write
\begin{align}
S_j &=  \sum_{i = j+1}^m \PBWunnorm(y_i, y_j \mid x)\\
&= \sum_{i = j+1}^m (F_i - F_{j-1})^N - (F_{i-1} - F_{j-1})^N - (F_{i} - F_j)^N + (F_{i-1} - F_j)^N\\
&= \sum_{i = j+1}^m (F_i - F_{j-1})^N - \sum_{i = j}^{m-1}( F_i - F_{j-1})^N - \sum_{i = j+1}^{m} (F_i - F_j)^N + \sum_{i = j}^{m-1} (F_i - F_j)^N\\
&= (F_m - F_{j-1})^N - (F_j - F_{j-1})^N - (F_m - F_j)^N \\
&= (1 - F_{j-1})^N - (1 - F_j)^N - (F_j - F_{j-1})^N
\end{align}
Therefore,
\begin{align}
    c_D &= \sum_{j = 1}^{m-1} S_j\\
    &= \sum_{j = 1}^{m-1} ((1 - F_{j-1})^N - (1 - F_j)^N - (F_j - F_{j-1})^N)\\
    &= \sum_{j = 0}^{m-2} (1 - F_j)^N - \sum_{j=1}^{m-1} (1 - F_j)^N - \sum_{j=1}^{m-1} (F_j - F_{j-1})^N\\
    &= 1 - (1 - F_{m-1})^N - \sum_{j=1}^{m-1} (F_j - F_{j-1})^N\\
    &= 1 - \sum_{j=1}^m (F_j - F_{j-1})^N
\end{align}
Similar to $c_I$, $c_D \to 1$ as $N$ grows assuming that the base distribution does not concentrate all probability on a single outcome.  There exists a constant $b > 0$ for which $c_D \geq b > 0$ for all $N$. \\

\textbf{Part 2: Bound the total variation between $\PBWI$ and $\PBW$.}\\
We will start with the total variation between the conditional distributions $\PBWI(y, y' \mid x)$ and $\PBW(y, y' \mid x)$. Let $\delta_{(m,1)}$ be the probability distribution that assigns probability $1$ to the pair $(y_m, y_1)$ and $0$ to all other pairs. By triangle inequality,
\begin{align}
\TV(\PBWI(y, y' \mid x), \PBW(y, y' \mid x)) &\leq \TV(\PBWI(y,y' \mid x), \delta_{(m,1)}) + \TV(\PBW(y, y' \mid x), \delta_{(m,1)})\\
&= (1 - \PBWI(y_m,y_1 \mid x)) + (1 - \PBW(y_m,y_1 \mid x))
\end{align}
The final line follows from $\delta_{(m,1)}$ put zero probability mass everywhere and assign probability $1$ to the pair $(y_m, y_1)$. Using the results from Part 1,
\begin{align}
    1 - \PBWI(y_m,y_1 \mid x) &= 1 - \frac{(F_m^N - F_{m-1}^N) ((1 - F_0)^N - (1 - F_1)^N)}{c_I}\\
    &= \frac{c_I - (1 - F_{m-1}^N)(1 - (1 - F_1)^N)}{c_I}\\
    &\leq \frac{1 - (1 - (1 - p_m)^N)(1 - (1 - p_1)^N)}{c_I}\\
    &= \frac{(1 - p_m)^N + (1 - p_1)^N - (1 - p_m)^N(1 - p_1)^N}{c_I}\\
    &\leq \frac{(1 - p_m)^N + (1 - p_1)^N}{c_I}
\end{align}
when $p_m$ and $p_1$ are the probability mass of $y_m$ and $y_1$ respectively.
Similarly,
\begin{align}
1 - \PBW(y_m,y_1 \mid x) &= 1 - \frac{(F_m - F_0)^N - (F_{m-1} - F_0)^N - (F_m - F_1)^N + (F_{m-1} - F_1)^N}{c_D} \\
&= \frac{c_D - (1 - (1 - p_m)^N - (1 - p_1)^N + (1 - p_m - p_1)^N)}{c_D}\\
&\leq \frac{1 - (1 - (1 - p_m)^N - (1 - p_1)^N + (1 - p_m - p_1)^N)}{c_D}\\
&= \frac{(1 - p_m)^N + (1 - p_1)^N - (1 - p_m - p_1)^N}{c_D}\\
&\leq \frac{(1 - p_m)^N + (1 - p_1)^N}{c_D}
\end{align}
Therefore, we have
\begin{align}
    \TV(\PBWI(y, y' \mid x), \PBW(y, y' \mid x)) &\leq \left(\frac{(1 - p_m)^N + (1 - p_1)^N}{c_I} + \frac{(1 - p_m)^N + (1 - p_1)^N}{c_D}\right)\\
&\leq 2(\frac{1}{a} + \frac{1}{b}) \max(1 - p_m, 1 - p_1)^N\\
&:= c_x \rho_x^N
\end{align}
Since we have this TV bound for the conditional distribution, we also have a similar bound for the total variation between $\PBWI$ and $\PBW$
\begin{align}
    \TV(\PBWI, \PBW) &\leq \sum_{x} \TV(\PBWI(y, y' \mid x), \PBW(y, y' \mid x)) P(x)\\
    &\leq \sum_{x} P(x)c_x \rho_x^N\\
    &\leq (\sum_{x} P(x) c_x) \max_x(\rho_x)^N\\
    &:= c \rho^N
\end{align}
This concludes the proof.
\end{proof}

\Qualityofrstarbw*
\begin{proof}
From Proposition \ref{prop:total-variation-bw-bwi}, we have that
\begin{equation}
    \TV(\PBW, \PBWI) \leq c\rho^N
\end{equation}
for some constant $c > 0$ and $\rho < 1$. We then apply the same argument as in the proof of Theorem \ref{thm:quality-rstarbr}, which allows us to bound the difference between $\rstarbw$ and $\rstarbwi$ in terms of the total variation between $\PBW$ and $\PBWI$. This concludes the proof.
\end{proof}

\section{Sample Complexity}
\label{app:sample-complexity}
In this section, we provide full proof details for results in Section \ref{sec:bon-theory}.

\Marginofrstarbi*

\begin{proof}
Recall the closed form for $\pBoN(y\mid x)$:
    \begin{equation}
    \label{eq:bon-closed-form-proof}
        \pBoN(y\mid x)
        = F(x,y)^N - F^{-}(x,y)^N
        = F(x,y)^N - \big(F(x,y)-\pbase\big)^N.
    \end{equation}
    
    We rewrite it as
    \begin{equation}
    \label{eq:bon-rewrite}
        \pBoN(y\mid x)
        = F(x,y)^N
          \Big(1 - \big(1 - \tfrac{\pbase(y\mid x)}{F(x,y)}\big)^N\Big).
    \end{equation}
    Let $\alpha_x(y) := 1 - \frac{P_{\text{base}}(y \mid x)}{F(x,y)} \in [0,1)$. Then
    \begin{equation}
    \label{eq:bon-alpha-form}
        \pBoN(y\mid x)
        = F(x,y)^N \Big(1 - \alpha_x(y)^N\Big).
    \end{equation}
    
    Substituting into the $\rstarbi$ expression (up to an additive constant):
    \begin{equation}
    \label{eq:bon-reward-expanded}
        \rstarbi(x,y)
        = \log\!\left(\frac{F(x,y)^N(1-\alpha_x(y)^N)}{\pbase}\right)
        = N\log F(x,y)
          + \log(1-\alpha_x(y)^N)
          - \log \pbase.
    \end{equation}
    
    Subtracting for $y$ and $y'$:
    \begin{align}
        \DeltarBI(x; y, y')
        &= N \big(\log F(x,y) - \log F(x,y')\big) + \big[ \log\big(1 - \alpha_x(y)^N\big) - \log\big(1 - \alpha_x(y')^N\big) \big]  - \log \left(\frac{\pbase}{\pbaseyprime}\right) \\
        &= N \big(\log F(x,y) - \log F(x,y')\big) + B_N(x; y, y')
    \end{align}
    
    Since $0 \leq \alpha_x(y) < 1$, we have $\alpha_x(y)^N \to 0$ as $N\to\infty$.
    The sequence $\{\log(1-\alpha_x(y)^N)\}_{N\ge1}$ is monotone increasing and bounded above by $0$ and below by $\log(1-\alpha_x(y))$.
    Therefore their difference remains bounded, and the final term $-\log(\pbase/\pbaseyprime)$ is constant. The dominant term is $N(\log F(x,y) - \log F(x,y'))$. If $r^*(x,y) > r^*(x,y')$, then $F(x,y) > F(x,y')$, so the margin diverges.
    \end{proof}

\Marginofrstarbwi*
\begin{proof}
We use the result from Theorem \ref{thm:identification-rstarbwi} which provides a closed form for $\rstarbwi$,
we have
\begin{equation}
    \rstarbwi(x,y) = \log \frac{\pBoN(y \mid x)}{\pWoN(y \mid x)} + c(x)
\end{equation}
and
\begin{align}
    \DeltarBWI(x;y,y') &= \rstarbwi(x,y) - \rstarbwi(x,y') \\
    &= \log \frac{\pBoN(y \mid x)}{\pWoN(y \mid x)} - \log \frac{\pBoN(y' \mid x)}{\pWoN(y' \mid x)} \\
&= \log \frac{F(x,y)^N - F^{-}(x,y)^N}{(1 - F^{-}(x,y))^N - (1 - F(x,y))^N} - \log \frac{F(x,y')^N - F^{-}(x,y')^N}{(1 - F^{-}(x,y'))^N - (1 - F(x,y'))^N} \\
&= \log \frac{F(x,y)^N ( 1 - (\frac{F^{-}(x,y)}{F(x,y)})^N)}{ (1 - F^{-}(x,y))^N ( 1 - (\frac{1 - F(x,y)}{1 - F^{-}(x,y)})^N)} - \log \frac{F(x,y')^N ( 1 - (\frac{F^{-}(x,y')}{F(x,y')})^N)}{ (1 - F^{-}(x,y'))^N ( 1 - (\frac{1 - F(x,y')}{1- F^{-}(x,y')})^N)}
\end{align}
We can see that the terms $( 1 - (\frac{F^{-}(x,y)}{F(x,y)})^N)$ and $1 - (\frac{1 - F(x,y)}{1 - F^{-}(x,y)})^N)$ are increasing with $N$ and converge to $1$ as $N \to \infty$. This also holds for $y'$. Therefore, we can write the margin as
\begin{align}
    \DeltarBWI(x;y,y') &= N \log \frac{F(x,y)}{1 - F^{-}(x,y)} - N \log \frac{F(x,y')}{1 - F^{-}(x,y')} + B_N(x;y,y')\\
    &= N \left(\log \frac{F(x,y)}{F(x,y')} + \log \frac{1 - F^{-}(x, y')}{1 - F^{-}(x,y)}\right) + B_N(x;y,y')
\end{align}
where $B_N(x;y,y')$ is uniformly bounded in $N$. This concludes the proof.
\end{proof}
\begin{theorem}[Formal connectivity bounds for $\PBI$]
    \label{thm:connectivity-bound-rstarbi-formal}
    Assume that every triplet in the support of $Q$ has non-zero probability under $\PBI$. Then there exists a constant $C_Q > 0$ for which for any $N$,
    \begin{equation}
        \lambdaconn(\PBI, Q ; \cH) \geq C_Q (\min_{x,y} \pbase)^N
    \end{equation}
    Fix $\alpha \in \mathbb{R}$ and define $A_{x, \alpha} = \{y \in \cY : r^*(x, y) \leq \alpha\}$. Assume that $\cH$ contains functions $f,g$ such that $f(x,\cdot) - g(x,\cdot)$ is constant on $A_{x,\alpha}$ and on its complement for every $x$, and that
    \[
        \gamma_{Q,\alpha} := \Pr_Q(y \in A_{x,\alpha}, y' \notin A_{x,\alpha}) + \Pr_Q(y \notin A_{x,\alpha}, y' \in A_{x,\alpha}) > 0.
    \]
    If $\eta := 1 - \max_{x,y} P_{\text{base}}(y \mid x) > 0$, then there exists a constant $D_{Q,\alpha} > 0$ such that for any $N$,
    \begin{equation}
        \lambdaconn(\PBI, Q ; \cH) \leq D_{Q,\alpha} \bE_{x \sim P_x}[P_{\text{base}}(A_{x, \alpha} \mid x)^N + P_{\text{base}}(A_{x, \alpha} \mid x)]
    \end{equation}
\end{theorem}

\begin{proof}
We start with the lower bound. Recall the definition of the connectivity degree,
\begin{equation}
    \lambdaconn(\PBI, Q ; \cH)  = \inf_{f,g \in \cH} \frac{\bE_{\PBI}[(\Delta f- \Delta g)^2]}{\bE_{Q}[(\Delta f- \Delta g)^2]}.
\end{equation}
If $\lambdaconn(\PBI, Q ; \cH) = 0$, the lower bound is immediate. So assume that $\lambdaconn(\PBI, Q ; \cH) > 0$. For any $f,g \in \cH$, we can rewrite
\begin{align}
    \bE_{\PBI}[(\Delta f- \Delta g)^2] &\geq \sum_{Q(x,y,y') > 0} \PBI(x,y,y')(\Delta f- \Delta g)^2\\
    &= \bE_{Q}[\frac{\PBI(x,y,y')}{Q(x,y,y')}(\Delta f- \Delta g)^2]\\
    &\geq \left(\min_{Q(x,y,y') > 0} \frac{\PBI(x,y,y')}{Q(x,y,y')}\right) \bE_{Q}[(\Delta f- \Delta g)^2].
\end{align}
Therefore,
\begin{equation}
    \lambdaconn(\PBI, Q ; \cH) \geq \min_{Q(x,y,y') > 0} \frac{\PBI(x,y,y')}{Q(x,y,y')}.
\end{equation}

Let
\begin{equation}
    Z_I(x) = 1 - \sum_{y \in \cY} \pBoN(y \mid x) P_{\text{base}}(y \mid x)
\end{equation}
be the normalization constant that removes the event $y = y'$. Then, for any $y \neq y'$, we have
\begin{align}
    \frac{\PBI(x,y,y')}{Q(x,y,y')} &= \frac{P(x) \pBoN(y \mid x) \pbaseyprime}{Z_I(x) Q(x,y,y')} \\
    &\geq \frac{P(x) \pBoN(y \mid x) \pbaseyprime}{Q(x,y,y')}
\end{align}
since $Z_I(x) \leq 1$. By Lemma \ref{lemma: bon-closed-form},
\begin{equation}
    \pBoN(y \mid x) = F(x,y)^N - F^{-}(x,y)^N \geq P_{\text{base}}(y \mid x)^N.
\end{equation}
Hence, for every $(x,y,y')$ in the support of $Q$,
\begin{equation}
    \frac{\PBI(x,y,y')}{Q(x,y,y')} \geq \frac{P(x) P_{\text{base}}(y' \mid x)}{Q(x,y,y')} P_{\text{base}}(y \mid x)^N \geq C_Q (\min_{x,y} \pbase)^N
\end{equation}
where
\begin{equation}
    C_Q := \min_{Q(x,y,y') > 0} \frac{P(x) P_{\text{base}}(y' \mid x)}{Q(x,y,y')} > 0.
\end{equation}
Substituting back, we obtain
\begin{equation}
    \lambdaconn(\PBI, Q ; \cH) \geq C_Q (\min_{x,y} \pbase)^N.
\end{equation}

Now, we prove the upper bound. Fix $\alpha \in \mathbb{R}$ and define
\begin{equation}
    A_{x, \alpha} = \{y \in \cY : r^*(x, y) \leq \alpha\}
\end{equation}
and write $p_A(x) := P_{\text{base}}(A_{x, \alpha} \mid x)$. Without loss of generality, let the set of outcomes be $\{y_1, y_2, \dots, y_m\}$ such that $r^*(x,y_1) < r^*(x,y_2) < \dots < r^*(x,y_m)$ (assuming that there is no tie in reward score). Let $k$ be the largest index for which $r^*(x,y_k) \leq \alpha$. Then
\begin{equation}
    p_A(x) = \sum_{i=1}^k P_{\text{base}}(y_i \mid x) = F(x, y_k).
\end{equation}
On the other hand, under the Best-of-$N$ sampling we have
\begin{equation}
    \pBoN(A_{x, \alpha} \mid x) = \sum_{i=1}^k \pBoN(y_i \mid x) = \sum_{i=1}^k (F(x,y_i)^N - F^{-}(x,y_i)^N)= F(x, y_k)^N = p_A(x)^N.
\end{equation}

By assumption, $\cH$ contains $f,g \in \cH$ such that for any $x$,
\begin{equation}
    f(x,y) - g(x,y) = \begin{cases}
        \gamma_1 & \text{if } y \in A_{x, \alpha}\\
        \gamma_0 & \text{otherwise.}
    \end{cases}
\end{equation}
Denote $\phi = f-g$. For this choice,
\begin{align}
    \lambdaconn(\PBI, Q ; \cH) &\leq \frac{\bE_{\PBI}[(\phi(x,y) - \phi(x,y'))^2]}{\bE_{Q}[(\phi(x,y) - \phi(x,y'))^2]}\\
    &= \frac{\Pr_{\PBI}(y \in \aalphax, y' \notin \aalphax ) + \Pr_{\PBI}(y \notin \aalphax, y' \in \aalphax )}{\Pr_{Q}(y \in \aalphax, y' \notin \aalphax ) + \Pr_{Q}(y \notin \aalphax, y' \in \aalphax )}.
\end{align}
By assumption, the denominator is exactly $\gamma_{Q,\alpha} > 0$. For the numerator, first note that
\begin{align}
    Z_I(x) &= 1 - \sum_{i=1}^{m} (F_i^N - F_{i-1}^N) p_i\\
    &\geq 1 - (\max_i p_i) \sum_{i=1}^{m} (F_i^N - F_{i-1}^N)\\
    &\geq 1 - \max_i p_i \\
    &\geq 1 - \max_{x,y} P_{\text{base}}(y \mid x) = \eta.
\end{align}
Therefore,
\begin{align}
    &\Pr_{\PBI}(y \in \aalphax, y' \notin \aalphax ) + \Pr_{\PBI}(y \notin \aalphax, y' \in \aalphax )\\
     &= \bE_{x \sim P_x}\left[\frac{1}{Z_I(x)}\left(\pBoN(A_{x, \alpha} \mid x)(1 - p_A(x)) + (1 - \pBoN(A_{x, \alpha} \mid x))p_A(x)\right)\right]\\
     &\leq \frac{1}{\eta} \bE_{x \sim P_x}\left[p_A(x)^N(1 - p_A(x)) + (1 - p_A(x)^N)p_A(x)\right]\\
     &\leq \frac{1}{\eta} \bE_{x \sim P_x}\left[p_A(x)^N + p_A(x)\right].
\end{align}
Substituting back, we conclude that
\begin{equation}
    \lambdaconn(\PBI, Q ; \cH) \leq \frac{1}{\eta \gamma_{Q,\alpha}} \bE_{x \sim P_x}[p_A(x)^N + p_A(x)].
\end{equation}
This proves the upper bound with $D_{Q,\alpha} = \frac{1}{\eta \gamma_{Q,\alpha}}$.
\end{proof}

\begin{theorem}[Connectivity bounds for $\PBWI$]
    \label{thm:connectivity-bound-rstarbiw}
    Let $P_{\text{base}}$ be a base distribution. Then there exists a constant $C > 0$ such that for any $N$,
    \begin{equation}
        \lambdaconn(\PBWI, Q ; \cH) \geq C \Big(\min_{x,y \neq y'} P_{\text{base}}(y \mid x) P_{\text{base}}(y' \mid x)\Big)^N
    \end{equation}
    On the other hand, for any real-valued $\alpha$, we define the set of outcomes for which the reward is less than $\alpha$ for a given context $x$ as $A_{x, \alpha} = \{y \in \cY : r^*(x, y) \leq \alpha\}$. For any $\alpha$, if $\cH$ is flexible enough, then there exists a constant $D > 0$ such that for any $N$,
    \begin{equation}
        \lambdaconn(\PBWI, Q ; \cH) \leq D \bE_{x \sim P_x}[1 +  (P_{\text{base}}(A_{x, \alpha} \mid  x))^N (1 - P_{\text{base}}(A_{x, \alpha} \mid x))^N]
    \end{equation}
\end{theorem}

\begin{proof}
    Here we absorb the normalization constant of $\PBWI$ into the constants. The proof is similar to the proof of Theorem \ref{thm:connectivity-bound-rstarbi-formal}, except that we replace the calculation of the probability mass of $\PBI$ with that of $\PBWI$. We start with the lower bound,

\begin{align}
    \frac{\PBWI(x,y,y')}{Q(x,y,y')} &= \frac{P(x) \pBoN(y \mid x) \pWoN(y' \mid x)}{Q(x,y,y')}\\
    &= \frac{P(x)(F(x,y)^N - F^{-}(x,y)^N) ((1 - F^{-}(x,y'))^N - (1 - F(x,y'))^N)}{Q(x,y,y')}\\
    &\geq \frac{P(x)\pbase^N\pbaseyprime^N}{Q(x,y,y')}.
\end{align}

Again, the term that decays with $N$ is $\pbase^N\pbaseyprime^N$. So there exists a constant $C > 0$ such that
\begin{equation}
    \frac{\PBWI(x,y,y')}{Q(x,y,y')} \geq C \Big(\min_{x,y \neq y'} P_{\text{base}}(y \mid x) P_{\text{base}}(y' \mid x)\Big)^N.
\end{equation}

For the upper bound, we first calculate the probability mass of $\pWoN(A_{x, \alpha} \mid x)$. We assume the same outcome ordering as in the proof of $\PBI$,
\begin{align}
\pWoN(A_{x, \alpha} \mid x) &= \sum_{i=1}^k \pWoN(y_i \mid x)\\
&= \sum_{i=1}^k ((1 - F^{-}(x,y_i))^N - (1 - F(x,y_i))^N)\\
&= 1 - (1 - F(x, y_k))^N \\
&= 1 - (1 - P_{\text{base}}(A_{x, \alpha}))^N.
\end{align}
With this, we have
\begin{align}
    &\Pr_{\PBWI}(y \in \aalphax, y' \notin \aalphax ) + \Pr_{\PBWI}(y \notin \aalphax, y' \in \aalphax )\\
     &= (\pBoN(A_{x, \alpha} \mid x))(1 - \pWoN(A_{x, \alpha} \mid x)) + (1 - \pBoN(A_{x, \alpha} \mid x))(\pWoN(A_{x, \alpha} \mid x))\\
     &= (P_{\text{base}}(A_{x, \alpha}))^N (1 - P_{\text{base}}(A_{x, \alpha}))^N + (1 - (P_{\text{base}}(A_{x, \alpha}))^N) (1 - (1 - P_{\text{base}}(A_{x, \alpha}))^N)\\
     &\leq  1 +  (P_{\text{base}}(A_{x, \alpha}))^N (1 - P_{\text{base}}(A_{x, \alpha}))^N.
\end{align}
\end{proof}

\begin{lemma}[General-$Q$ estimation error]
\label{lem:general-q-estimation-error}
Let $P$ be a comparison distribution, let $\cH$ be a convex hypothesis class bounded by some constant $B$, and let $\rstarp \in \arg\min_{r \in \cH} \LP(r)$. Let $S$ be a set of $n$ i.i.d.\ samples drawn from $P$, and let $\rhat$ be the empirical risk minimizer of the BT learning objective based on $S$. Then there exists a constant $D > 0$ such that with probability at least $1 - \delta$,
\begin{equation}
    \bE_Q[(\Delta_{\rhat} - \Deltarstarp)^2] \leq D \frac{\operatorname{Comp}(n, \cH, \delta)}{\lambdaconn(P, Q ; \cH)}
\end{equation}
where
\begin{equation}
    \operatorname{Comp}(n, \cH, \delta) := \frac{1}{\sigma_B}\left(\hat{\mathfrak{R}}_S(\cHpair) + M_B\sqrt{\frac{\log(2/\delta)}{n}}\right),
\end{equation}
with $M_B := \log(1 + \exp(2B))$ and $\sigma_B := \sigma(2B)(1 - \sigma(2B))$.
\end{lemma}

\begin{proof}
If $\lambdaconn(P, Q ; \cH) = 0$, the claim is vacuous. So assume that $\lambdaconn(P, Q ; \cH) > 0$. By the definition of the connectivity degree,
\begin{equation}
    \bE_Q[(\Delta_{\rhat} - \Deltarstarp)^2] \leq \frac{1}{\lambdaconn(P, Q ; \cH)} \bE_P[(\Delta_{\rhat} - \Deltarstarp)^2].
\end{equation}
By Lemma \ref{lemma: margin semi-norm is bounded by population risk difference},
\begin{equation}
    \bE_P[(\Delta_{\rhat} - \Deltarstarp)^2] = \lVert \rhat - \rstarp \rVert_{\Delta, P} \leq \frac{2}{\sigma_B}(\LP(\rhat) - \LP(\rstarp)).
\end{equation}

Let $\widehat{\LP}$ denote the empirical BT risk on the sample $S$. Since $\rhat$ is the empirical risk minimizer,
\begin{align}
    \LP(\rhat) - \LP(\rstarp) &= (\LP(\rhat) - \widehat{\LP}(\rhat)) + (\widehat{\LP}(\rhat) - \widehat{\LP}(\rstarp)) + (\widehat{\LP}(\rstarp) - \LP(\rstarp))\\
    &\leq 2 \sup_{r \in \cH} |\LP(r) - \widehat{\LP}(r)|.
\end{align}
Now let $\ell(a) = -\log \sigma(a)$. Since $\cH$ is bounded by $B$, we have $|r(x,y) - r(x,y')| \leq 2B$, so
\begin{equation}
    \ell(a) \leq \log(1 + \exp(2B)) := M_B
\end{equation}
for all relevant $a$. Moreover, $\ell$ is $1$-Lipschitz. Therefore, by the standard Rademacher-complexity bound and the contraction lemma, with probability at least $1 - \delta$,
\begin{equation}
    \sup_{r \in \cH} |\LP(r) - \widehat{\LP}(r)| \leq 2 \hat{\mathfrak{R}}_S(\cHpair) + 3M_B\sqrt{\frac{\log(2/\delta)}{n}}.
\end{equation}
Combining the last three displays, we obtain
\begin{equation}
    \bE_Q[(\Delta_{\rhat} - \Deltarstarp)^2] \leq \frac{1}{\lambdaconn(P, Q ; \cH)} \left(\frac{8}{\sigma_B}\hat{\mathfrak{R}}_S(\cHpair) + \frac{12M_B}{\sigma_B}\sqrt{\frac{\log(2/\delta)}{n}}\right).
\end{equation}
Absorbing the numerical factors into a constant $D$ proves the lemma.
\end{proof}

\AccuracyBound*

\begin{proof}
Since $\rstarbi$ preserves the ranking induced by $\rstar$, the sign-based accuracy with respect to $\rstar$ is the same as the sign-based accuracy with respect to $\rstarbi$. By Lemma \ref{lem:general-q-estimation-error} applied with $P = \PBI$ and $\rstarp = \rstarbi$, there exists a constant $D > 0$ such that with probability at least $1 - \delta$,
\begin{equation}
    \bE_Q[(\Delta_{\rhat} - \DeltarBI)^2] \leq D \frac{\operatorname{Comp}(n, \cH, \delta)}{\lambdaconn(\PBI, Q ; \cH)}.
\end{equation}

For any threshold $u > 0$, order preservation implies
\begin{align}
    \operatorname{Acc}_Q(\rhat) &\geq \Pr_Q(|\Delta_{\rhat} - \DeltarBI| \leq |\DeltarBI|)\\
    &\geq \Pr_Q(|\DeltarBI| \geq u) - \Pr_Q(|\Delta_{\rhat} - \DeltarBI| > u)\\
    &\geq \Pr_Q(|\DeltarBI| \geq u) - \frac{\bE_Q[(\Delta_{\rhat} - \DeltarBI)^2]}{u^2}\\
    &\geq \Pr_Q(|\DeltarBI| \geq u) - D \frac{\operatorname{Comp}(n, \cH, \delta)}{u^2\lambdaconn(\PBI, Q ; \cH)},
\end{align}
where the third line follows from Markov's inequality.

Next, we use Proposition \ref{thm:margin-rstarbi}. Because $\cX$ and $\cY$ are finite and every triplet in the support of $Q$ has non-zero probability mass under $P_{\text{base}}(\cdot \mid x)$, the uniform boundedness of $B_N(x;y,y')$ implies that
\begin{equation}
    M_Q := \max_{(x,y,y') \in \operatorname{supp}(Q)} \sup_{N \geq 1} |B_N(x;y,y')| < \infty.
\end{equation}
For every $(x,y,y')$ in the support of $Q$, Proposition \ref{thm:margin-rstarbi} gives
\begin{equation}
    \DeltarBI(x;y,y') = N \log \frac{F(x,y)}{F(x,y')} + B_N(x;y,y').
\end{equation}
Therefore,
\begin{equation}
    |\DeltarBI(x;y,y')| \geq N \left|\log \frac{F(x,y)}{F(x,y')}\right| - M_Q.
\end{equation}
In particular, for any $k > 0$,
\begin{equation}
    \left|\log \frac{F(x,y)}{F(x,y')}\right| \geq k + \frac{M_Q}{N}
    \quad \Longrightarrow \quad
    |\DeltarBI(x;y,y')| \geq Nk.
\end{equation}
Hence,
\begin{equation}
    \Pr_Q(|\DeltarBI| \geq Nk) \geq \Pr_Q\left(\left|\log \frac{F(x,y)}{F(x,y')}\right| \geq k + \frac{M_Q}{N}\right).
\end{equation}

Substituting $u = Nk$ into the accuracy bound above yields
\begin{align}
    \operatorname{Acc}_Q(\rhat) \geq \sup_{k > 0}
    \Big[
        \Pr_Q\left(\left|\log \frac{F(x,y)}{F(x,y')}\right| \geq k + \frac{M_Q}{N}\right)
        -
        D \frac{\operatorname{Comp}(n, \cH, \delta)}{N^2k^2\lambdaconn(\PBI, Q ; \cH)}
    \Big].
\end{align}
This is the desired result.
\end{proof}

\section{Experiment details}
\label{app:experiment-details}

In this section, we provide full experimental details and additional results that supplement the main experiments. We conduct experiments on synthetic data to validate our theoretical findings, following the setup of \citet{pukdee2026doespreferencelearningrecover}.

\subsection{Synthetic Data Setup}

\textbf{Context and Response Space.}
Let $\mathcal{S} = \{u_1, u_2, \dots, u_m\}$ be a finite set of items with $u_i \in \mathbb{R}^d$, and let $\mathcal{X} = \mathcal{Y} = \mathcal{S}$. We use the same space for contexts and responses to simplify the synthetic setting, so the score function can be interpreted as a similarity function between items. We set $m = 16$ and $d = 128$ for simplicity. To generate the dataset, we sample each $u_i$ from $\mathcal{N}(0, I_d)$.

\textbf{Ground-truth Score Function.}
The ground-truth score is defined via a two-layer neural network $f^*: \mathbb{R}^d \to \mathbb{R}^{e}$ with hidden dimension $h = 32$ and output embedding dimension $e = 8$, using ReLU activations:
\begin{equation}
    r^*(x, y) = \operatorname{cosine-similarity}(f^*(x), f^*(y)),
\end{equation}
ensuring scores lie in $[-1, 1]$.

\textbf{Triplet Generation.} We generate $n$ triplets $S = \{(x_i, y_i^+, y_i^-)\}_{i=1}^n$ by sampling $x_i$ uniformly from $\mathcal{X}$ and drawing $y_i^+$ and $y_i^-$ from a negative distribution $p^-(\cdot \mid x_i)$ and a positive distribution $p^+(\cdot \mid x_i)$ that is BT-consistent with $r^*$, i.e., $\log(p^+(y \mid x) / p^-(y \mid x)) = r^*(x, y)$. By default, we use a uniform negative distribution $p^-(\cdot \mid x) = \frac{1}{m}$. The validation set is also generated in the same way.

\textbf{Training and Evaluation.} We train a two-layer neural network with the same architecture as $f^*$ to minimize the empirical BT loss on $S$. We use the Adam optimizer with learning rate selected from $\{10^{-4}, 10^{-3}, 10^{-2}\}$. We train for 200 epochs and select hyperparameters via a validation set of 2048 triplets drawn from the same distribution based on validation loss. We report accuracy with respect to the uniform distribution $Q$ over all pairs $(u_i, u_j)$, with standard error over 5 random seeds.

\clearpage

\subsection{Comparison between Best-vs-Random and Best-vs-Worst for each $N$}

\begin{figure*}[h]
    \centering
    \begin{subfigure}[b]{0.32\linewidth}
        \includegraphics[width=\linewidth]{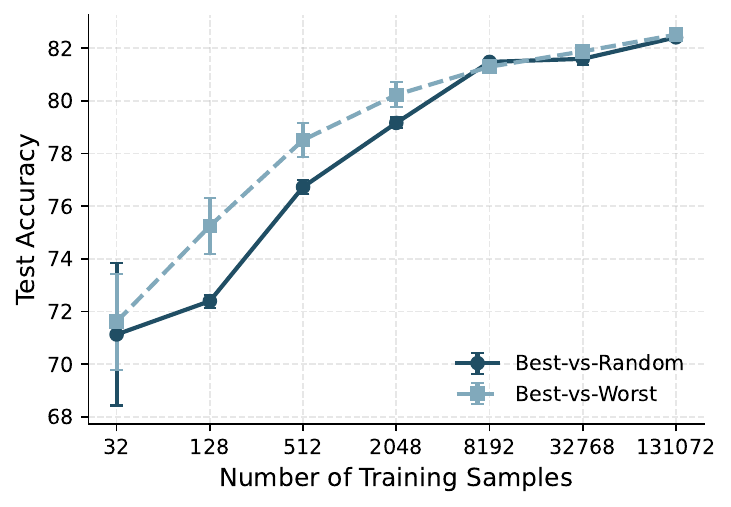}
        \caption{$N = 4$}
    \end{subfigure}
    \hfill
    \begin{subfigure}[b]{0.32\linewidth}
        \includegraphics[width=\linewidth]{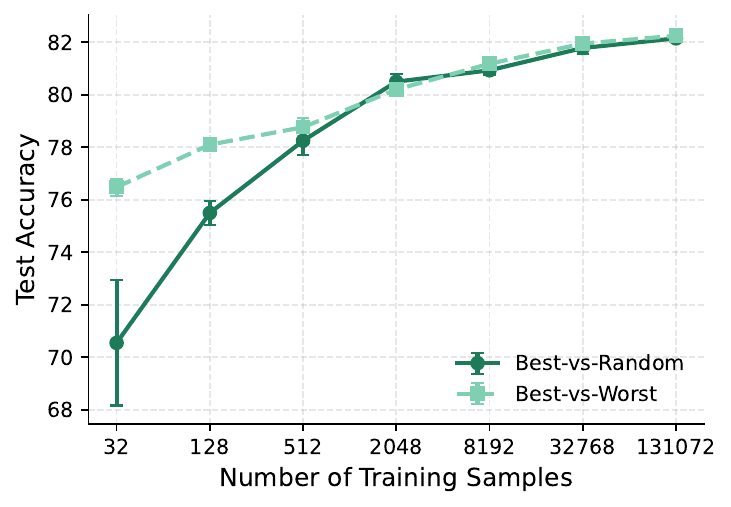}
        \caption{$N = 8$}
    \end{subfigure}
    \hfill
    \begin{subfigure}[b]{0.32\linewidth}
        \includegraphics[width=\linewidth]{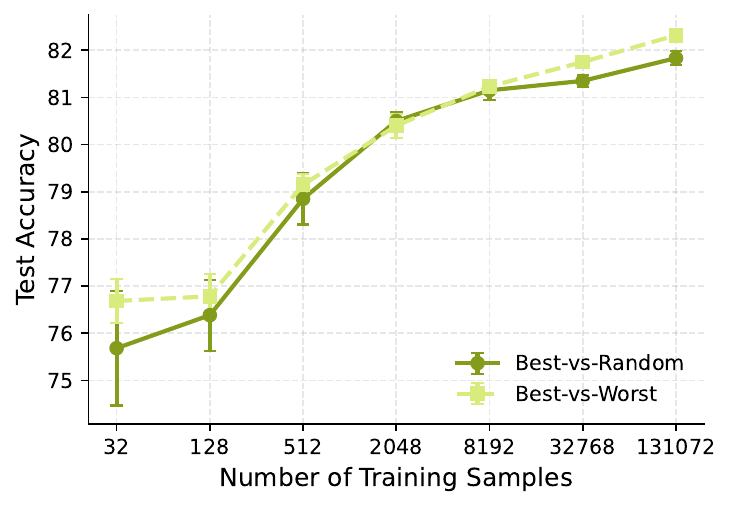}
        \caption{$N = 16$}
    \end{subfigure}
    \caption{Test accuracy for Best-vs-Random and Best-vs-Worst on UltraFeedback for each $N \in \{4, 8, 16\}$.}
\end{figure*}

\subsection{Model ablation}
\label{app:model-ablation}

\begin{figure*}[h]
    \centering
    \begin{subfigure}[b]{0.32\linewidth}
        \centering
        \includegraphics[width=\linewidth]{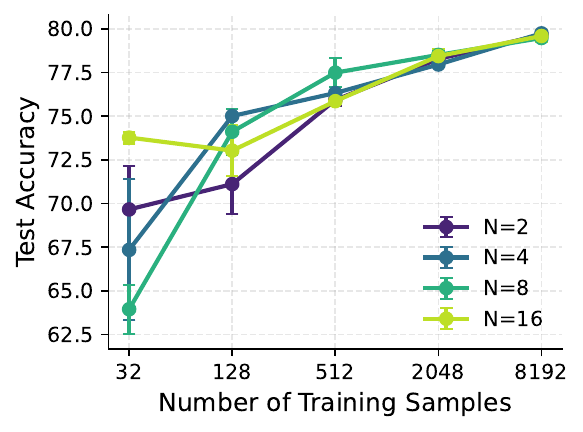}
        \caption{\texttt{gemma-3-12b-it}}
    \end{subfigure}
    \hfill
    \begin{subfigure}[b]{0.32\linewidth}
        \centering
        \includegraphics[width=\linewidth]{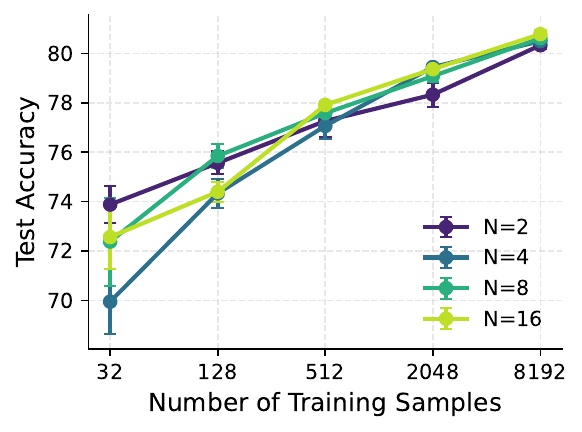}
        \caption{\texttt{Ministral-3-8B-Instruct-2512}}
    \end{subfigure}
    \hfill
    \begin{subfigure}[b]{0.32\linewidth}
        \centering
        \includegraphics[width=\linewidth]{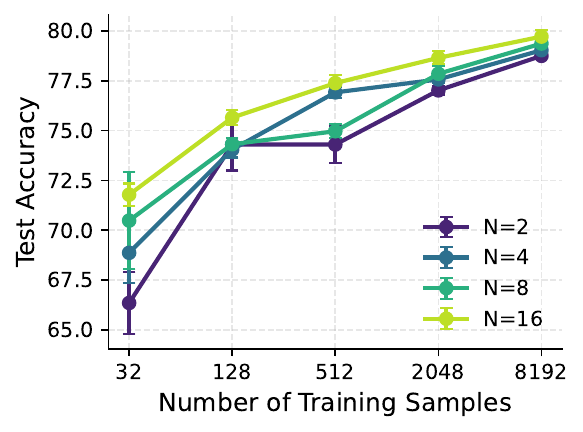}
        \caption{\texttt{Qwen3-8B}}
    \end{subfigure}
    \caption{Test accuracy for Best-vs-Random across different base models.}
\end{figure*}

\subsection{Connectivity degree of UltraFeedback}
\begin{figure}[h]
    \centering
    \includegraphics[width=0.5\textwidth]{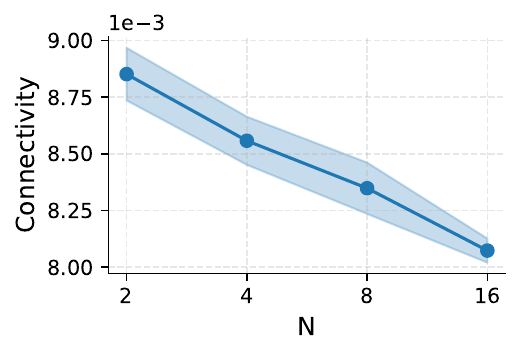}
    \caption{Connectivity vs. $N$ for \texttt{UltraFeedback}.}
    \label{fig:connectivity-vs-n-ultrafeedback}
\end{figure}

\clearpage
\subsection{Additional results for synthetic data}
\label{app:additional-results-synthetic}
\begin{figure*}[h]
    \centering
    \begin{subfigure}[b]{0.45\linewidth}
        \includegraphics[width=\linewidth]{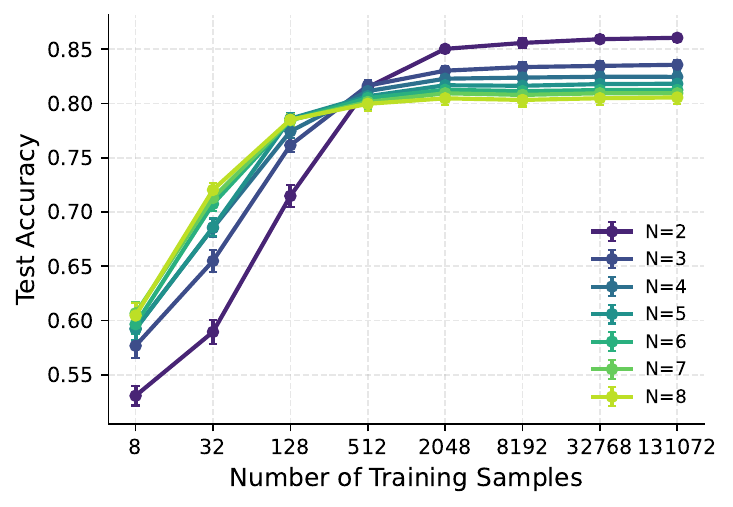}
        \caption{Accuracy}
    \end{subfigure}
    \hfill
    \begin{subfigure}[b]{0.45\linewidth}
        \includegraphics[width=\linewidth]{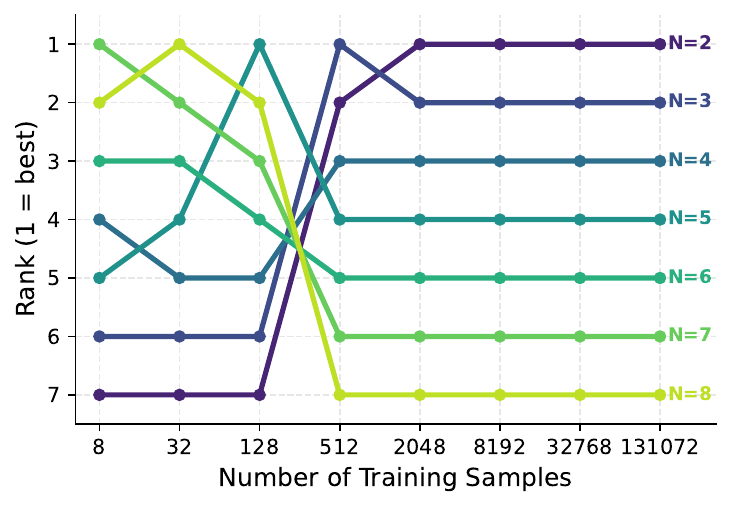}
        \caption{Ranking}
    \end{subfigure}
    \caption{Test accuracy for Best-vs-Worst-Independent on synthetic data.}
\end{figure*}

\begin{figure}[h]
    \centering
    \includegraphics[width=\textwidth]{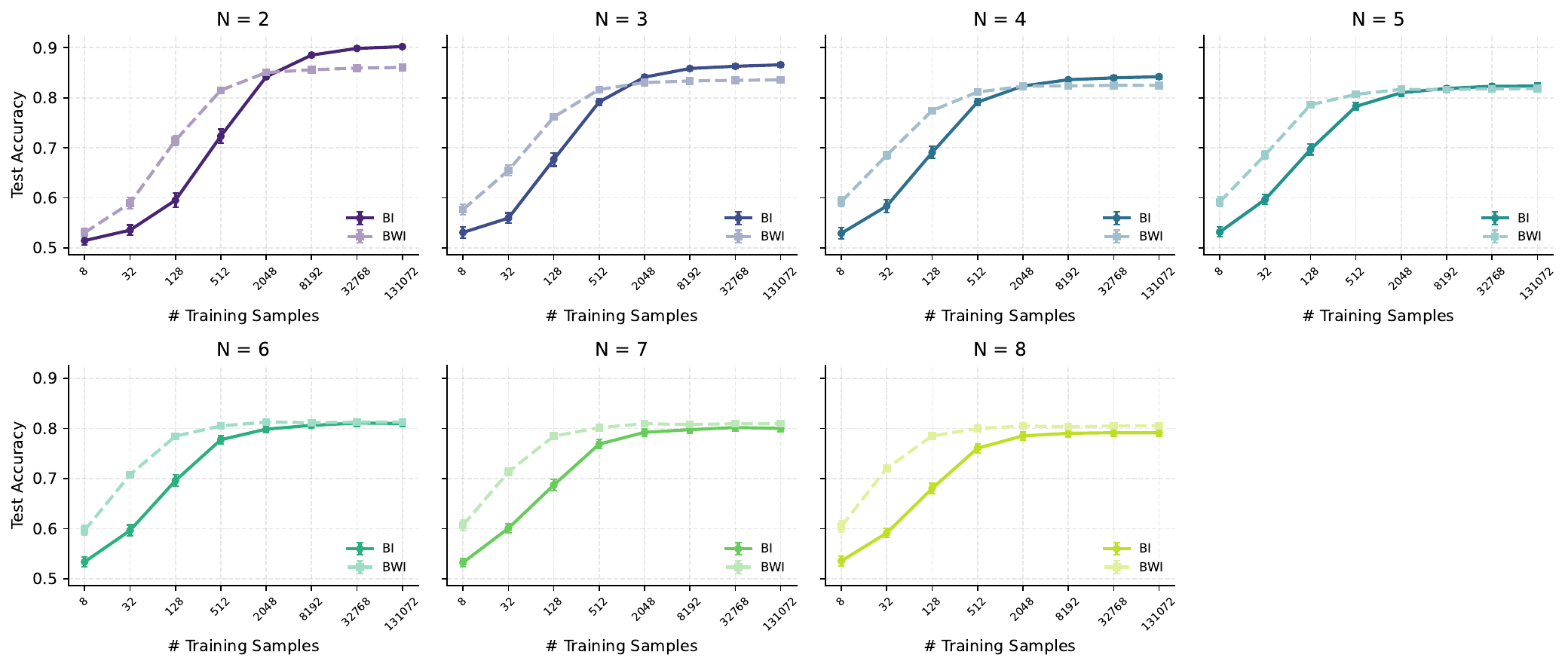}
    \caption{Test accuracy for Best-vs-Independent vs. Best-vs-Worst-Independent for each $N$.}
    \label{fig:bon-vs-won-per-n}
\end{figure}

\begin{figure}[h]
    \centering
    \includegraphics[width=0.40\textwidth]{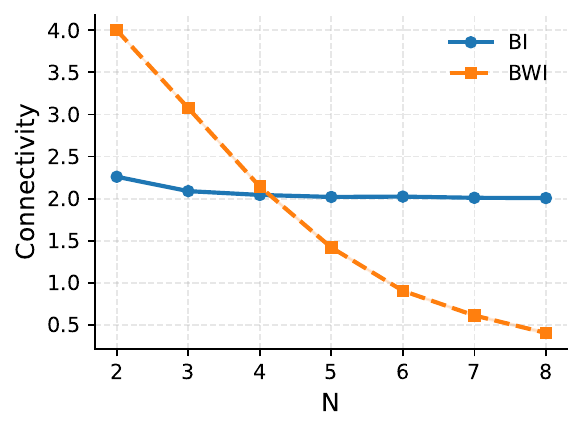}
    \caption{Connectivity degree for Best-vs-Independent vs. Best-vs-Worst-Independent for each $N$.}
    \label{fig:bon-vs-won-connectivity}
\end{figure}

\subsection{Details on distribution reshaping}
\label{app:distribution-reshaping}

To modify the shape of the base distribution in a standardized way, we use upsampling and downsampling. Each time we sample a response from the base distribution, we accept it with probability $1-p$. Otherwise, we evaluate its reward score. If the score lies in the percentile range $[c - w/2, c + w/2]$ for the given context $x$, then we accept the response in the upsampling setting and reject it in the downsampling setting; outside this range, we do the opposite. We repeat this process until we obtain $N = 4$ responses, and then apply the Best-vs-Random mechanism to generate the training data. Equivalently, this procedure increases or decreases the probability mass in a specified score band, yielding the modified base distributions
\begin{equation}
    P_{\text{base,up}} \propto (1 - p) P_{\text{base}} + p P_{c, w}
\end{equation}
\begin{equation}
    P_{\text{base,down}} \propto (1 - p) P_{\text{base}} - p P_{c, w}
\end{equation}
where $P_{c, w}$ denotes the base distribution with probability mass concentrated in the region whose reward score lies in the percentile range $[c - w/2, c + w/2]$. Here $c$ is the percentile center and $w$ is the width of the band. We use $c \in \{15, 30, 50, 70, 85\}$, $w = 30$, and $p = 0.8$. This up/downsampling procedure provides a standardized way to reshape the base distribution. For example, setting $c = 15$ targets lower-score responses, while setting $c = 85$ targets higher-score responses. We use $n = 512$ samples for each dataset and train a reward model with the same setup as in the previous experiments.

\subsection{Distribution of the reward score under different base distributions}

\begin{figure}[h]
    \centering
    \includegraphics[width=0.9\textwidth]{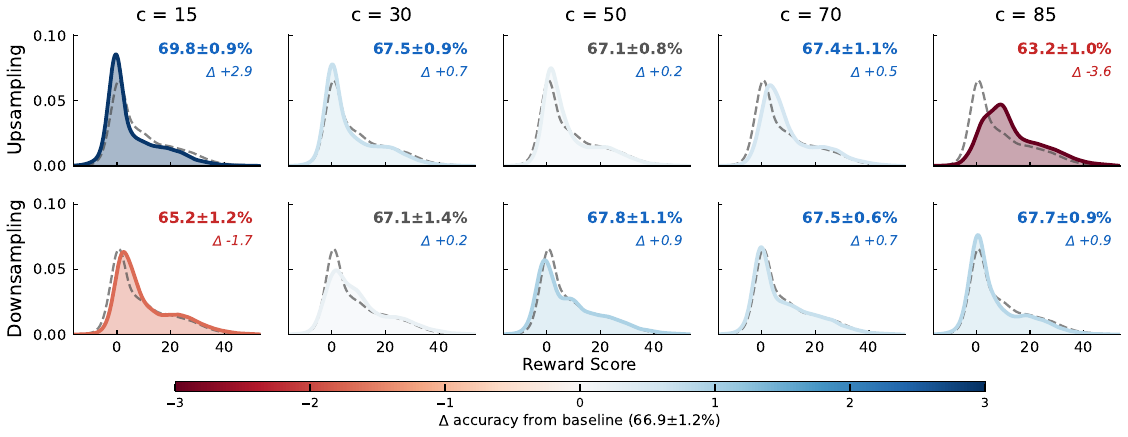}
\caption{Distribution of the reward score for different base distributions for \texttt{PKU-SafeRLHF}.}
\end{figure}

\begin{figure}[h]
    \centering
    \includegraphics[width=0.9\textwidth]{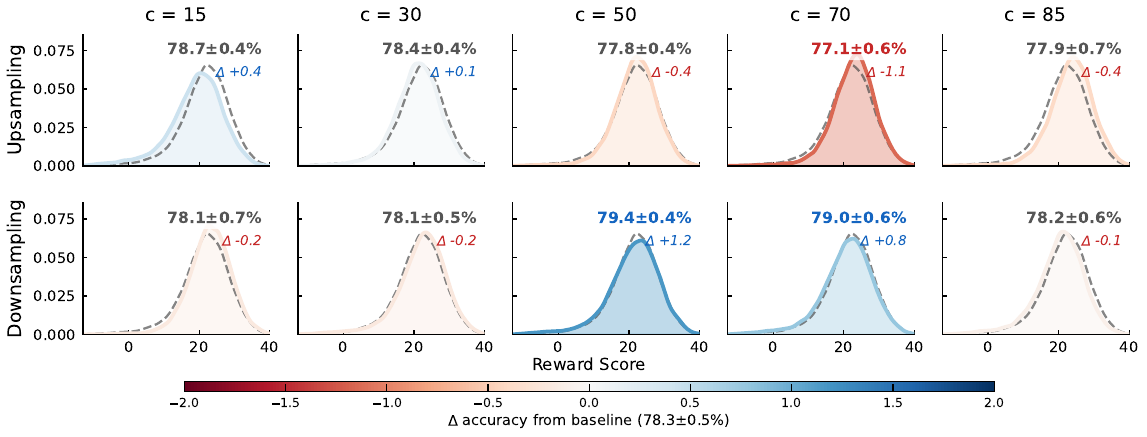}
\caption{Distribution of the reward score for different base distributions for \texttt{GSM8K}.}
\end{figure}

\begin{figure}[h]
    \centering
    \includegraphics[width=0.9\textwidth]{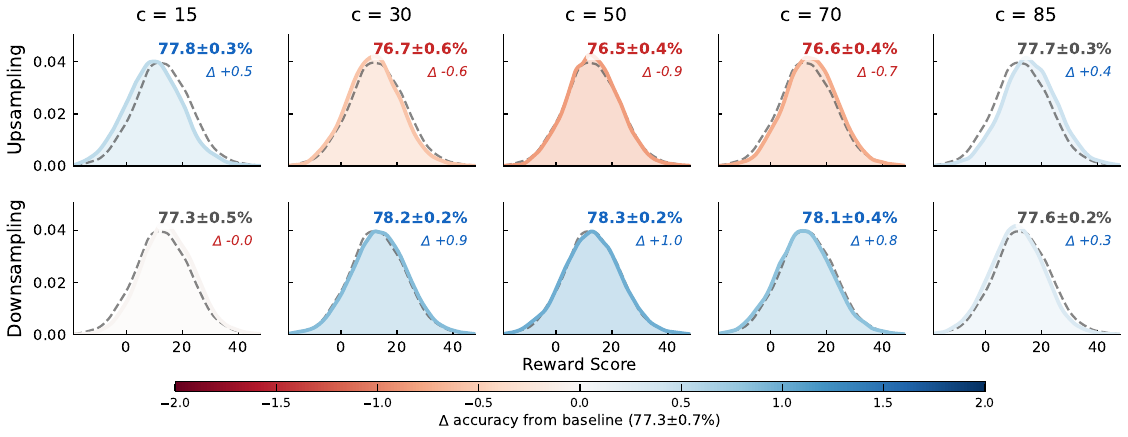}
\caption{Distribution of the reward score for different base distributions for \texttt{UltraFeedback}.}
\end{figure}

\clearpage
\subsection{Additional details for GSM8K}
\label{app:additional-plots-gsm8k}

\texttt{GSM8K} \citep{cobbe2021gsm8k} is a question-answering dataset rather than a preference dataset, so we construct preference pairs from model-generated responses. For each prompt in the test set, we sample multiple responses from the base generator and score them with the reward model. Each response is labeled as correct or incorrect by extracting the final numeric answer and comparing it to the gold answer. We then pair a randomly selected correct response against the highest-scoring incorrect response (i.e., the most convincing wrong answer according to the reward model), and discard prompts that lack either a correct or incorrect response.

A complication is that the reward model is poorly calibrated for mathematical correctness on \texttt{GSM8K}: the score distributions for correct and incorrect responses overlap substantially, meaning a non-trivial fraction of correct responses receive lower reward scores than incorrect ones from the same prompt. To avoid contaminating the test set with pairs where the reward model effectively disagrees with the ground-truth labeling, we additionally require the chosen (correct) response to outscore the rejected (incorrect) response by a reward margin of at least 2. Figure~\ref{fig:gsm8k-scores-distribution} shows the resulting score distributions: even after this filter, the chosen and rejected distributions remain heavily overlapping, indicating that the test set is far from trivial.

\begin{figure}[ht]
    \centering
    \begin{subfigure}[t]{0.48\textwidth}
        \centering
        \includegraphics[width=\linewidth]{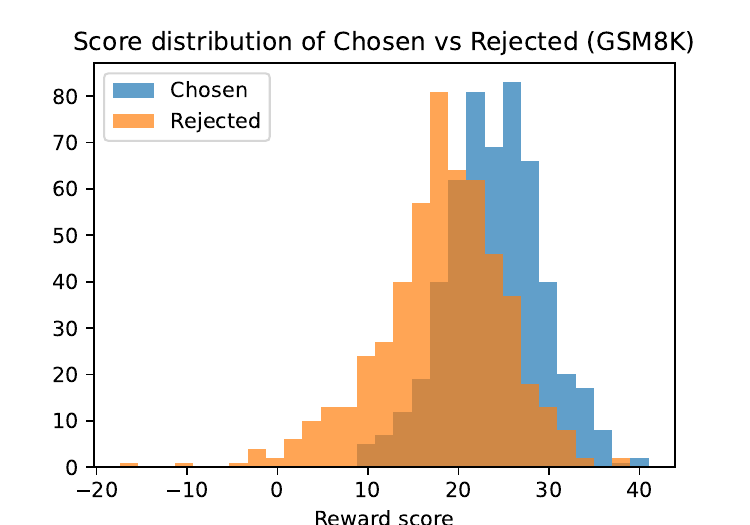}
        \caption{Distribution of the reward score for \texttt{GSM8K}.}
        \label{fig:gsm8k-scores-distribution}
    \end{subfigure}\hfill
    \begin{subfigure}[t]{0.48\textwidth}
        \centering
        \includegraphics[width=\linewidth]{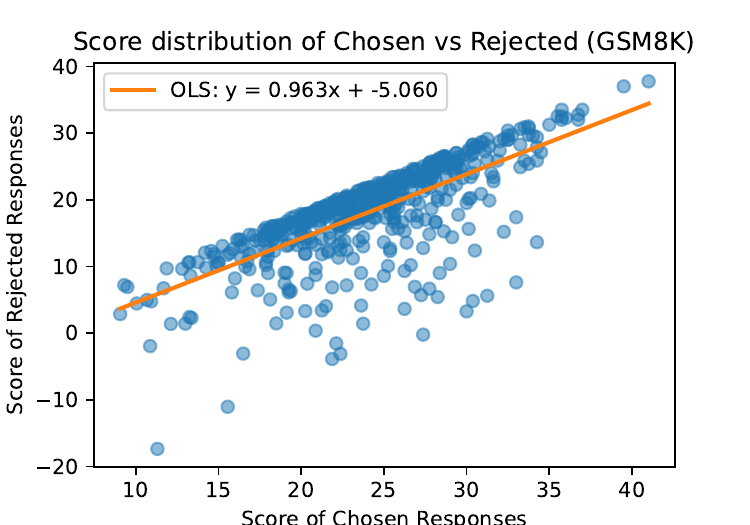}
        \caption{Scatter plot of the reward score for \texttt{GSM8K}.}
        \label{fig:gsm8k-scores-distribution-scatter}
    \end{subfigure}
    \caption{Reward scores on \texttt{GSM8K}: histogram and scatter.}
    \label{fig:gsm8k-scores-combined}
\end{figure}

\clearpage
\end{document}